\documentclass[a4paper,fleqn]{cas-sc}

\usepackage[authoryear,longnamesfirst]{natbib}

\def\tsc#1{\csdef{#1}{\textsc{\lowercase{#1}}\xspace}}
\tsc{WGM}
\tsc{QE}

\newtheorem{theorem}{Theorem}
\newtheorem{lemma}[theorem]{Lemma}
\newtheorem{proposition}[theorem]{Proposition}
\newdefinition{remark}{Remark}
\newdefinition{definition}{Definition}
\newdefinition{example}{Example}
\newdefinition{problem}{Problem}
\newproof{proof}{Proof}
\newproof{pot}{Proof of Theorem \ref{thm}}

\usepackage{graphicx}%
\usepackage{multirow}%
\usepackage{amsmath,amssymb,amsfonts}%
\usepackage{mathrsfs}%
\usepackage{dsfont}%
\usepackage{subcaption} 
\usepackage{enumitem}  
\usepackage{booktabs} 
\usepackage{tabularx} 
\usepackage{array}    

\newcommand{\R}{\mathbb R}
\newcommand{\N}{\mathbb N}
\newcommand{\Z}{\mathbb Z}

\newcommand{\ep}{\varepsilon}

\newcommand{\LL}{\mathcal{L}}

\newcommand{\FF}{\widehat{F}}
\newcommand{\GG}{\widehat{G}}
\newcommand{\PP}{\tilde{P}}

\newcommand{\dotProduct}[2]{\langle #1, #2 \rangle}
\newcommand{\DD}{\mathscr{D}}

\DeclareMathOperator{\ind}{ind}

\DeclareMathOperator{\supp}{supp}

\DeclareMathOperator{\id}{id}
\DeclareMathOperator{\dom}{dom}
\DeclareMathOperator{\img}{img}

\DeclareMathOperator{\co}{co}

\begin{document}
\let\WriteBookmarks\relax
\def\floatpagepagefraction{1}
\def\textpagefraction{.001}

\shorttitle{Knot-Preserving Directional Mollification}

\shortauthors{A. González-Calvin, J. F. Jiménez, and H. G. de Marina}

\title [mode = title]{Directional Mollification for Knot-Preserving
\texorpdfstring{$C^{\infty}$}{C-infinity} Smoothing of Polygonal Chains with
Explicit Curvature Bounds}


%

\author[1]{Alfredo González-Calvin}[
orcid=0009-0004-2760-7096]

\cormark[1]


\ead{alfredgo@ucm.es}



\affiliation[1]{organization={Dept. of Computer Architecture and Automation, Faculty
            of Physics, Complutense University of Madrid},
            addressline={Plaza de las Ciencias, 1},
            city={Madrid},
            postcode={28040},
            state={Madrid},
            country={Spain}}

\author[2]{Juan F. Jiménez}[orcid=0000-0002-1448-2296]


\ead{juan.jimenez@fis.ucm.es}



\affiliation[2]{organization={Dept. of Computer Architecture and Automation, Faculty
            of Physics, Complutense University of Madrid},
            addressline={Plaza de las Ciencias, 1},
            city={Madrid},
            postcode={28040},
            state={Madrid},
            country={Spain}}

\author[3]{Héctor García de Marina}[orcid=0000-0002-3341-6745]


\ead{hgdemarina@ugr.es}



\affiliation[3]{organization={Dept. of Computer Engineering, Automation, and
            Robotics\& Institute of Mathematics, University of Granada},
            addressline={Calle Ventanilla},
            city={Granada},
            postcode={18001},
            state={Granada},
            country={Spain}}

\cortext[1]{Corresponding author}



\begin{abstract}                          
We introduce the \textit{directional mollification} operator, which acts on
polygonal chains to produce $C^{\infty}$ curve approximants that are arbitrarily
close to the original curve---pointwise and uniformly on compact subsets---while strictly interpolating its vertices. Unlike
standard mollification, which fails to preserve vertices, this directional
construction enables local, vertex-preserving smoothing; modifying a single
segment alters the $C^{\infty}$ output only within an explicitly controllable
neighborhood. The operator admits closed-form curvature bounds and provides
analytic control over curve geometry. Furthermore, we develop a parametric
family of smoothing operators that unifies conventional and directional
mollification into a single framework. Combining computational simplicity with
strict waypoint fidelity, this method is directly applicable to geometric
modeling, robotics, computer graphics, and CNC machining.
\end{abstract}


\begin{highlights}
    \item Directional mollification enables the easy generation, with formal
    convergence, curvature and locality guarantees, of $C^{\infty}$
    curves that preserve every vertex of a polygonal chain.
    \item Closed-form computationally lightweight curvature bounds for
    directionally mollified polygonal chains are derived, enabling direct feasibility
    checks for geometric design and motion-control applications.
    \item Formal computations show that the directional mollification of a
    polygonal chain and itself coincide in several easy-to-compute sets centered
    in the midpoint of the segments.
    \item A single-parameter family of smoothing operators is introduced,
    encompassing both mollification and the proposed directional variant as
    special cases.
    \item Qualitative visual comparisons with classical spline schemes are
    presented to illustrate the distinctive geometric behavior of the proposed method.
\end{highlights}


\maketitle

\section{Introduction}

Given an ordered set of vertices (also called knots), it is a
typical desired task to construct a $C^r$ curve, with $r > 1$ such that (i)
interpolates the given vertices exactly, (ii) remains arbitrarily close (in an
appropriate sense of convergence) to the original piecewise-linear polygonal
curve defined by the points, (iii) admits closed-form and efficiently computable
curvature bounds, and (iv) exhibits strong locality, such that modification of a single segment affects only a small neighborhood of the curve. These
requirements arise, for example, in geometric modeling, high-precision curve
fitting, trajectory generation in industrial and mobile robotics, CNC machining
and computer graphics.

Classical approaches based on spline constructions and subdivision schemes
address many of these properties but involve trade-offs. For instance, spline
and Bézier representations provide compact parametric descriptions and local
control \cite{piegl2012nurbs, farin2002curves, cohen2001geometric}. However,
enforcing exact interpolation while maintaining high smoothness often requires
knot insertion or degree elevation \cite{schumaker2007spline, hasan2024b,
numericalstabilitybezierBspline}. Subdivision schemes, such as the Refine-Smooth
algorithm \cite{HAMEED2017289} or
corner cutting \cite{TAN2015819, DEBOOR1987125, hormann2024curvature, DYN2022102083},
show that, under certain constraints, the designed curve can be
continuously differentiable, even curvature continuous, or obtain
$C^{5}$ continuity, but lack vertex interpolation.
Gaussian process regression
\cite{rasmussen2006} and kernel smoothing \cite{hastie2009} offer $C^{\infty}$
smoothness, but sacrifice exact knot interpolation and introduce $O(n^3)$
computational overhead alongside potential numerical instability.
Optimization-based formulations \cite{mellinger2011,flory2008constrained,bentamy2005cross} embed
curvature constraints within convex programs, supporting real-time execution,
yet their performance is sensitive to parameter selection and requires careful
tuning \cite{heiden2018grips, ohrhallinger20212d}. There have also been efforts in
\cite{CHEN2022103408,LinearPathsIndustrialRobots, Tajima_2021ijat} to create smooth
vertex-preserving curves from polygonal chains, but they usually require
high-degree elevation or cumbersome computations. For example, B-splines are used in
\cite{hua2023global} to control the approximation error and the curvature of the
smoothed curve in CNC machining, thus generating knot-preserving feasible
curves. Nevertheless, it comes at a cost of degree elevation and $C^r$
continuity with $r < \infty$. Recently, Pythagorean-hodograph (PH)
curves have been shown to provide exact arc-length parametrization,
closed-form curvature expressions, and efficient construction of
rotation-minimizing frames \cite{farouki2008pythagorean,
farouki2017comprehensive}. PH quintic splines support local modification and
Hermite interpolation while preserving key geometric invariants
\cite{BAY2023128240,FAROUKI2026102505}. However, PH constructions typically rely on algebraic
constraints on hodographs and fixed polynomial degree, which can complicate
flexibility and extension to arbitrary data distributions.

The technique of mollification---convolution of a discontinuous or nonsmooth
function with a compactly supported $C^{\infty}$ kernel
\cite{Evans2022-PDE}---offers a conceptually and computationally simple,
analytically rigorous method to obtain $C^{\infty}$ smoothness, without the
overhead of degree elevation of polynomial splines. Applied to polygonal chain
curves, conventional mollification has been recently introduced in
\cite{gonzalezcalvin2025efficientgenerationsmoothpaths}. This technique
produces, from polygonal chains and locally-integrable functions, $C^{\infty}$
(locally, i.e., at vertices, and globally) approximants with closed-form
equations, as well as upper bounds on curvature, while remaining computationally
efficient and numerically stable. Nevertheless, a fundamental geometric
limitation of conventional mollifiers is their confinement of the image of the
mollified curve to the convex hull of the original data. This implies that
the locations of the vertices of a polygonal chain cannot, in general, be preserved by
mollification alone. For many geometric-design tasks, as well as robotics and CNC
specifications---where exact vertices preservation is an explicit
specification---this convex-hull restriction is unacceptable.

In this work, we remove the previous restriction by introducing a new term in the
mollification operator that, for polygonal chains, produces a vertex-invariant
curve with the same analytical tractability, convergence, smoothness and numerical efficiency of conventional mollification. The proposed directional
mollifier improves the conventional mollification with a convolution-weighted
mean of directional derivatives of the original curve. This construction avoids
the convex-hull restriction of conventional mollification, and, by a suitable
choice of a parameter $\ep > 0$, the operator yields $C^{\infty}$ curve approximants that:
(i) pass exactly through every original knot of a polygonal chain, (ii)
converge---letting $\ep$ to $0$---pointwise and uniformly in compact subsets to any
almost everywhere differentiable function, (iii) admit explicit closed-form
curvature bounds, and (iv) display strong locality: changing a single segment
perturbs the smoothed curve only on that segment and in a controllably small
neighborhood of its endpoints. As an example, see Figure \ref{fig:Figure_1_Example},
where the original polygonal chain is shown in black, and the conventional and
directional mollifications in green and red, respectively,
for a value of the aforementioned parameter $\ep=0.4$. As can be seen, the conventional mollification is restricted to the convex hull (of the image of)
the original function, and does not intersect its vertices. Nevertheless, the
directional mollification presented in this work intersects its vertices, and as we will show preserves the
smoothness and convergence properties of conventional mollifiers. Finally, as
our last contribution, we introduce a weight on the directional derivative term
yielding a parametric family of $C^{\infty}$ functions sharing the same
convergence properties, of which both conventional mollification and directional mollification are special instances. The resulting theory provides simple,
closed-form tools for generating $C^{\infty}$ knot-preserving functions,
analytical and rigorous curvature guarantees, and local interaction, all of them
properties of immediate interest in geometric modeling, curve design, mobile and
industrial robotics and CNC machining. Therefore, it complements both subdivision schemes and novel PH curve
methodologies by trading the exact polynomial structure and arc-length
parametrization for greater flexibility, analytic smoothness, and a unifying operator-based formulation suitable for a broader class of geometric processing
tasks.

This work is organized as follows. Section \ref{sec:Notation} sets notation and
preliminaries. Section \ref{sec:Mollifiers} reviews conventional mollifiers, presenting key results, and recalls several properties of (right) directional derivatives. Section
\ref{sec:NewMollification} defines the directional mollification, proving
smoothness and convergence properties and discussing several differences with
respect to the conventional mollification for curve smoothing. Section
\ref{sec:ApplicationThreePointTwoSegments} applies the theory to polygonal
chains, proving vertex-preservation, curvature bounds, and highlighting locality
(single-segment edits affect only a small neighborhood) and practical parameter
choices. Section \ref{sec:Combination} presents the parametric family unifying
conventional and directional mollifications. Section \ref{sec:Conclusions}
summarizes the contributions and concludes this work.

\section{Notation}\label{sec:Notation}

We define as $\N=\{1,2,\dots,\}$ the set of positive integers, and as $\N_0 = \N \cup
\{0\}$. All integrals must be thought of with respect to the Lebesgue measure and the Borel sigma
algebra. For a function $f \in L^p(X):=L^p(X,\R)$ with $p \in [1,\infty]$, we
denote the $L^p$ norm as $\|f\|_p$ while for a vector $x\in \R^n$ its $\ell_p$
norm as $\|x\|_p$. It will be clear from the context their difference. We say
that  $f : X \to \R$ is locally $p$-integrable, denoted $f \in
L^p_{\mathrm{loc}}(X,\R)$, if $f$ is $p$-integrable on every compact subset of
$X$.  We denote by $\id : X \to X$ the identity function defined as $\id(x) =
x$.  The indicator (or characteristic) function of a set $A$ is denoted
as $\ind_{A}$.
If $f : X \to Y$ and $A \subset X$, then we denote the restriction of $f$ to $A$
as $f|_{A}$. For any two sets $X,Y$ we denote $C(X,Y)$ as the set of continuous
functions from $X$ to $Y$ and for $n \in \N \cup \{\infty\}$ we denote as
$C^n(X,Y)$ the set of $n$-times continuously differentiable functions from $X$
to $Y$. If $Y = \R$, then we will usually note $ C^n(X):=C^n(X,\R)$.  For a set $A$
we denote its closure as $\overline{A}$, and its convex hull as $\co(A)$. The
support of a real-valued function is defined as $\supp f := \overline{\{x \in
\dom f \mid f(x) \neq 0\}}$, where $\dom f$ is the domain of $f$, and
we denote as $\img(f)$ the image of a function. If $X \subset
\R^n$ is open, we also denote the set $\DD(X)$ as the set of functions that
belong to $C^{\infty}(X)$ and have compact support in $X$. If $f : \R^n \to \R$
and $g : \R^n \to \R$ are two functions, we denote their convolution, whenever it
exists, as $f * g$. A curve in $\R^n$ is a measurable function $f : X
\subset \R \to \R^n$. Writing $f = (f_1,\dots,f_n)$, we call $f_i$ the $i$'th
component of the curve for $i \in \{1,\dots,n\}$, and denote its derivative as
$Df_i$ when it exists. If $\varphi : \R \to \R$, we denote by $f * \varphi := (f_1 *
\varphi, \dots, f_n * \varphi)$ their convolution component-wise whenever it exists.
We say $f \in L^1_{\mathrm{loc}}(X,\R^n)$ if each of its
components is an $L^1_{\mathrm{loc}}(X,\R)$ function.
Moreover, if a function is defined
as $f:\R^n \to \R$, $x = (x_1,\dots,x_n) \mapsto f(x)$ then $\partial_{x_i} f$
or $\partial_{i}f$ is the partial derivative of $f$ with respect to the $i$th
variable. If $\alpha = (\alpha_1,\dots,\alpha_n) \in \N_0^n$ is a multi-index, we
denote $\partial_x^{\alpha}f =
\partial_{x_1}^{\alpha_1}\dots\partial_{x_n}^{\alpha_n}f=
\partial_{1}^{\alpha_1}\dots\partial_{n}^{\alpha_n}f$. We denote by
$\dotProduct{\cdot}{\cdot} : \R^n \times \R^n \to \R$ the standard inner product
in $\R^n$, and by $B(0,1) = \{x \in \R^n \mid \|x\|_2 < 1\}$ and $\overline{B}(0,1) =
\{x \in \R^n \mid \|x\|_2 \leq 1\}$ the unit balls in $\R^n$ with the Euclidean
norm. Given a continuous function $f \in C([a,b],\R^n)$ we define its length
as $\mathcal{L}(f)$ as the supremum
of $\sum_{P}\|f(t_i)-f(t_{i-1})\|$
taken over all the partitions $P$
of $[a,b]$, and $\|\cdot\|$ is any norm in $\R^n$.
Finally, given a  set of ordered points, called knots or vertices,
$(P_i)_{i=0}^{p} \subset \R^n\times \cdots \times \R^n$ with $p \in \N$, we say
that $f : [0,p] \to \R^n$ is a continuous $p+1$ knots/vertices polygonal chain
curve if $f$ consists of the line segments joining the consecutive vertices.

\section{Mollifiers, convex functionals and generalized directional derivatives}
\label{sec:Mollifiers}

\subsection{Mollifiers}
The regularization properties of mollifiers are deeply introduced in
\cite{gonzalezcalvin2025efficientgenerationsmoothpaths}. Here we present their
definitions, summarize several properties, and provide an example.
\begin{definition}[Mollifier]\label{def:MollifiersDefinition}
    Let $\varphi \in \DD(\R^n)$ and for $\ep >0$ define
    $\varphi_{\ep} := \frac{1}{\ep^n}\varphi \circ
    \frac{\id}{\ep}$. We call $\varphi$ a mollifier if it satisfies:
    \begin{enumerate}
        \item $\int_{\R^n}\varphi = 1$, and
        \item in the distributional sense $\lim_{\ep\to 0} \varphi_{\ep} = \delta$,
        where $\delta$ is the Dirac delta distribution.
    \end{enumerate}
\end{definition}

Let us present one of the most popular mollifiers, since it will be used
extensively in this paper.
\begin{example}\label{example:OurMollifier}
    Let $\varphi : \R \to [0,\infty)$ be the function
    \begin{equation}\label{eq:OurMollifier}
        \varphi(x) = c_1 \exp\left(\frac{-1}{1-x^2}\right)\ind_{(-1,1)}(x),
    \end{equation}
    where $c_1 > 0$ is a normalization constant that ensures $\int_{\R}\varphi
    = 1$. Clearly $\supp \varphi = \overline{(-1,1)} =
    [-1,1]$ and $\supp \varphi_{\ep} = [-\ep,\ep]$. Moreover, with a change of
    variables it can be seen that $\int_{\R}\varphi_{\ep} =
    1$, and it can also be shown using the Lebesgue Dominated Convergence
    Theorem that as $\ep \to 0^+$ the second property in
    Definition \ref{def:MollifiersDefinition} holds.
\end{example}

\begin{figure}[pos=t]
    \centering
    \includegraphics[width=0.5\linewidth]{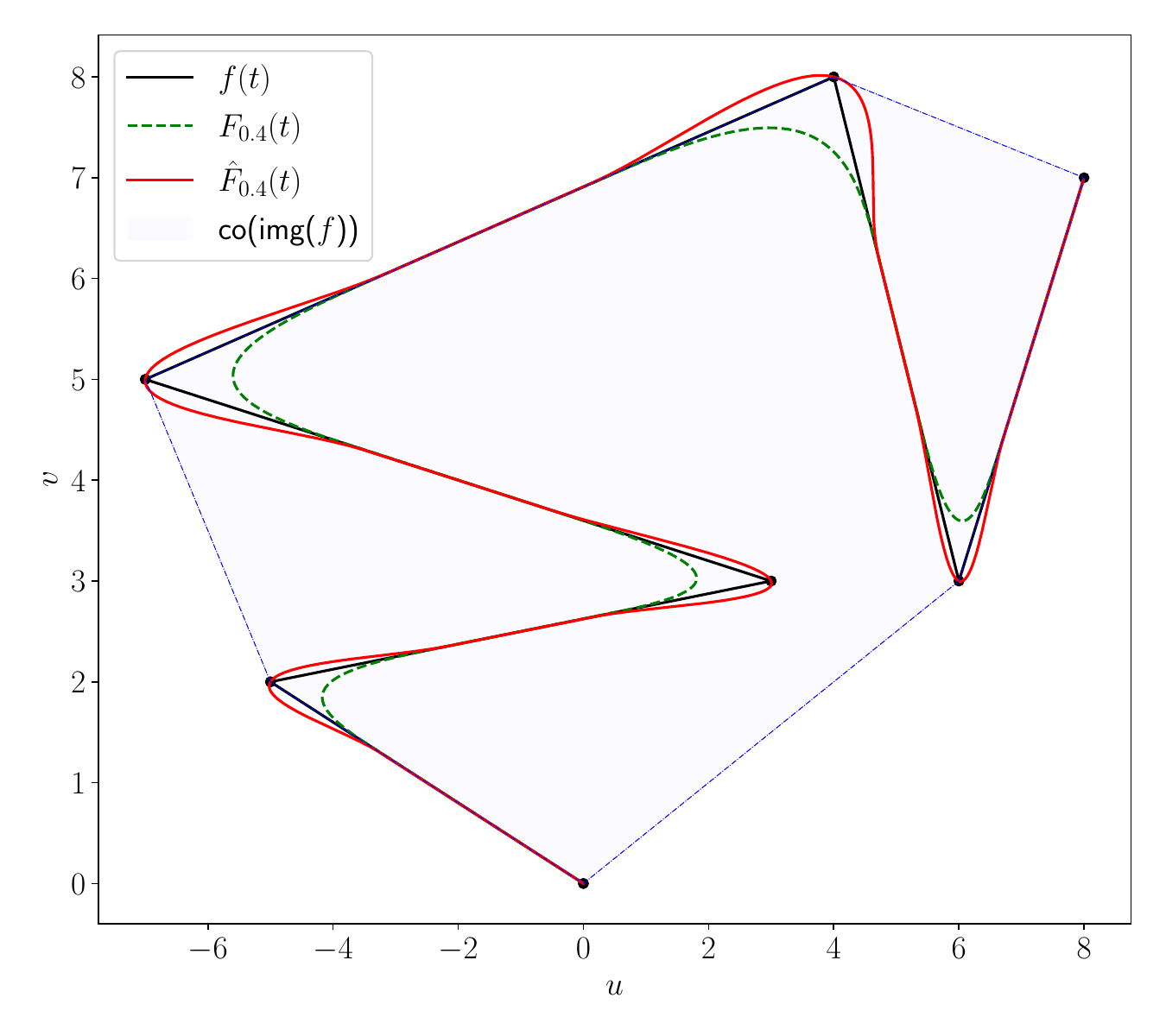}
    \caption{Example of mollification of a polygonal chain $f : \R \to \R^2$,
    where the coordinates in $\R^2$ are expressed as $(u,v)$.
    The green curve shows the (conventionally) mollified function $F_{0.4}$  as
    in Theorem \ref{thm:OurReferenceTheorem} with the mollifier presented in Example
    \ref{example:OurMollifier}. The number $0.4$, stands
    for a parameter $\ep=0.4$ in which the mollification depends. The original (black) curve is an illustrative case of a polygonal chain curve whose knots/vertices are
    shown as black dots. As  can be seen, the mollified (green) curve is
    contained in the convex hull (blue filled polygon) of the image of the
    original curve as Theorem \ref{thm:OurReferenceTheorem} states, but it does
    not intersect the vertices. This is a trade-off between smoothness, convergence,
    and the approximation provided by the conventional mollification with
    \eqref{eq:OurMollifier}. Nevertheless, it will be shown in Section
    \ref{sec:ApplicationThreePointTwoSegments} that for this kind of curve, our directional mollification method will generate a smooth curve that inherits
    the convergence and smoothness properties of the conventional mollification
    while intersecting \emph{all} the knots/vertices that define the
    original polygonal chain curve, while obtaining explicit curvature estimates. This is
    illustrated by the red plot, with the directional mollification $\FF_{0.4}$
    introduced and studied in this paper.}
    \label{fig:Figure_1_Example}
\end{figure}
\begin{theorem}[{\cite[Appendix C, Theorem
    7]{Evans2022-PDE}}]\label{thm:PropertiesOfMollifying}
    Let $f \in L_{\mathrm{loc}}^p(\R^n,\R)$ with $p \in [1,\infty]$, let $\varphi \in
    \DD(\R^n)$ be a mollifier, and let $\ep > 0$.
    The following three statements hold:
    \begin{enumerate}
        \item $f*\varphi_{\ep} \in C^{\infty}(\R^n)$ and for any $n \in \N$
            it holds $\partial^{\alpha}(\varphi_{\ep} * f) =
            (\partial^{\alpha}\varphi_{\ep}) * f$ for all multi index $\alpha
            \in \N_0^n$.
        \item $\varphi_{\ep} * f \to f$ pointwise almost everywhere as $\ep \to
            0^+$.
        \item If $f$ is continuous then $f * \varphi_{\ep}\to f$ as $\ep \to
            0^+$ on compact subsets of $\R^n$.
    \end{enumerate}
\end{theorem}
And we sum up the results presented in
\cite{gonzalezcalvin2025efficientgenerationsmoothpaths} using a single
theorem.
\begin{theorem}[\cite{gonzalezcalvin2025efficientgenerationsmoothpaths}]
\label{thm:OurReferenceTheorem}
    Let $f \in L_{\mathrm{loc}}^{p}(\R^n,\R), \varphi \in \DD(\R^n, \R)$
    be a mollifier, and define $F_{\ep} := f * \varphi_{\ep}$.
    \begin{enumerate}
        \item If $f$ is convex (resp. concave)  and $\varphi \geq 0$ then
        $F_{\ep} \geq f$ and $F_{\ep}$ is convex (resp.  $F_{\ep} \leq f$
        and $F_{\ep}$ is concave) for any $\ep > 0$. If $f$ is quasiconvex
        and $\varphi \geq 0$ then $F_{\ep}$ is also quasiconvex, for all $\ep >
        0$.
        \item If $f$ is locally convex (resp. locally concave) on an open
        convex set $U \subset \R^n$ and $\varphi \geq 0$, then for any open
        convex set $V
        \subset U$ there exists a $\delta > 0$ such that for all $\epsilon
        \in (0,\delta)$ it holds that $F_{\ep}|_{V} \geq f|_{V}$ and
        $F_{\ep}|_{V}$ is convex (resp. $F_{\ep}|_{V} \leq f|_{V}$ and
        $F_{\ep}|_{V}$ is concave).
        \item If $f(x) = \dotProduct{a}{x} + c$ with $c \in \R$,
        $a \in \R^n$ and $\varphi$ is an even function whose support
        is $\supp \varphi = \overline{B}(0,1)$ then it holds that
        $F_{\ep} = f$ for all $\ep  > 0 $.
        \item If $n = 1$ and $f$ is monotone, then
        $F_{\ep}$ is monotone for all $\ep > 0$.
    \end{enumerate}
    Moreover, if $f \in C([a,b],\R^n)$ with $-\infty < a < b < \infty$ and
    $\varphi \in \DD(\R)$ is a non negative mollifier whose support is $[-1,1]$,
    for $F_{\ep} := (f_i * \varphi_{\ep})_{i=1}^{n}$ it holds that
    \begin{enumerate}
        \item $\mathcal{L}(F_{\ep}) \leq \LL(f)$ for all $\ep > 0$, where $\LL
        : C(\R,\R^n) \to \R$
        is the length of the curve $f$, which has been extended as done
        in \cite[Lemma 7 and Theorem
        8]{gonzalezcalvin2025efficientgenerationsmoothpaths}, computed
        with respect to any norm in $\R^n$.
        \item $F_{\ep}([a,b]) \subset \co(f([a,b]))$, for all $\ep > 0$.
    \end{enumerate}
\end{theorem}

See Figure \ref{fig:Figure_1_Example} for an illustration of the conventional mollification
and the directional mollification that will be introduced and studied in this
paper.

\subsection{Convex functionals and directional derivatives}

Before introducing the directional mollification, we need to present the definition of the (right) directional derivative, and a theorem that will be useful in the following sections.
\begin{definition}
    Let $f : \R^n \to \R^m$ be a function. Given $x, h \in \R^n$ the
    directional derivative of $f$ at $x$ in the direction of $h$ is
    defined as
    \begin{equation*}
        f^{\circ}(x)(h) := \lim_{t \to
            0^+}\frac{f(x+th)-f(x)}{t},
    \end{equation*}
    provided the limit exists. If the limit exists for all $h \in \R^n$, then
    $f$ is said to be directionally differentiable at $x$.
\end{definition}

\begin{theorem}[{\cite[Lemma 3.3 and Theorem
l    3.4]{jahn2007introduction}}]\label{thm:ConvexityProperties}
    If $f : \R^n \to \R$ is convex then it is directionally differentiable
    at all $x \in \R^n$ and, given $y \in \R^n$, then $f(x) \geq f(y) +
    f^{\circ}(x)(x-y)$.
    If $f$ is
    differentiable almost everywhere (not necessarily convex), then
    for any $y \in \R^n$, $f^{\circ}(x)(y) = \dotProduct{\nabla f(x)}{y}$ for
    almost all $x \in \R^n$.
    The same statements hold, but with reverse inequalities for concave
    functions.
\end{theorem}

\section{Constructing the directional mollification}\label{sec:NewMollification}

\subsection{The improvement of conventional mollification}
As can be seen from Figure \ref{fig:Figure_1_Example}, the mollification of a
polygonal chain stays inside the convex hull of the original curve. This implies that, for a
convex real-valued function, its mollification is above (i.e., greater than or equal
to) it. As presented in Theorem \ref{thm:OurReferenceTheorem}, the convex hull property also holds
for any locally integrable function, not necessarily restricted to
polygonal chains. Moreover, we can also see that, for the
polygonal chain, the mollified curve does \textit{not} intersect the points that
define the curve; thus they cannot be considered as vertices for the
smoothed $C^{\infty}$ function. Nevertheless, as seen in Theorem
\ref{thm:OurReferenceTheorem}, the smoothness, convergence, and computational
aspects of mollification are highly desirable properties for curve smoothing.
For that purpose, we collect all of these desirable properties and state them
in a single problem to be solved in this work.
\begin{problem}[Improving the conventional mollification]\label{prob:NewRegProblem}
    Given $p \in [1,\infty]$ and a $p$-locally integrable function $f : \R^n \to \R$
    or $f : \R \to \R^n$, find an operator $T_{\ep}$, where $\ep > 0$, that acts
    on $f$ and creates a new function $T_{\ep}(f)$ with the same domain and
    codomain of $f$ such that:
    \begin{enumerate}
        \item $T_{\ep}(f)$ inherits the smoothness properties of
        mollifiers, that is, $T_{\ep}(f)$ is infinitely many times continuously
        differentiable.
        \item $T_{\ep}(f)$ can be made arbitrary close to $f$ by letting $\ep
        \to 0$. That is, it inherits the convergence properties of the
        mollifiers.
        \item If $f : \R^n \to \R$ is convex, then $T_{\ep}(f) \leq f$.
        \item If $f : \R \to \R^n$ is a polygonal chain in $\R^n$, the smoothed
        curve generated, $T_{\ep}(f):\R \to \R^n$ shall intersect all the vertices
        that define the chain.
        \item $T_{\ep}(f)$ is simple to compute numerically.
    \end{enumerate}
\end{problem}

\subsection{Solving Problem \ref{prob:NewRegProblem}}
From now on, we will impose several restrictions on the mollifier, for which we introduce the following function space.
\begin{definition}
    We define $\DD_M(\R^n)$ as the set of mollifiers that are nonnegative and whose support is $\overline{B}(0,1)$. That is,
    \begin{equation*}
        \DD_M(\R^n) := \{\varphi \in \DD(\R^n) \text{ mollifier } \mid \varphi \geq 0 \text{ and } \supp\varphi = \overline{B}(0,1)\}.
    \end{equation*}
\end{definition}

We are going to solve Problem \ref{prob:NewRegProblem} by adding to the mollified function a new term
that we call, the \textit{directional derivative term.}
\begin{definition}[Directional derivative term]\label{def:DirDerTerm}
    Let $f \in L_{\mathrm{loc}}^p(\R^n, \R)$ and $\varphi \in
    \DD_M(\R^n)$. For $x \in \R^n$
    define $f^{\circ}_x : \R^n \to \R$ as $s \in \R^n \mapsto f^{\circ}_x(s) :=
    f^{\circ}(s)(x-s)$ and let $f^{\circ}_x \in
    L^{q}_{\mathrm{loc}}(\R^n,\R)$ for some $1 \leq q \leq \infty$. We define the directional
    derivative term for $\ep > 0$, $D_{\ep} : \R^n \to \R$, as
    \begin{equation*}
        D_{\ep}(x) :=
        \int_{\R^n}f^{\circ}_x(s)\varphi_{\ep}(x-s)ds =
        \int_{\R^n}f^{\circ}(s)(x-s)\varphi_{\ep}(x-s)ds.
    \end{equation*}
\end{definition}
\begin{remark}
    Note that for any $\ep > 0$ and $x \in \R^n$, $D_{\ep}(x)$  can be thought
    as a weighted mean of the directional derivatives centered at $x \in \R^n$.
\end{remark}

We finally introduce the concept of \textit{directional mollification}.
\begin{definition}[Directional mollification]\label{def:NewMollification}
    Let $p \in [1,\infty]$, $f \in L^p_{\mathrm{loc}}(\R^n,\R)$ and $\varphi \in
    \DD_M(\R^n)$. Let $f$ be directionally differentiable at any
    point, and for any $x \in \R^n$ let $f^{\circ}_x \in L^{q}_{\mathrm{loc}}(\R^n,\R)$
    for some $q \in [1,\infty]$.  For $\ep > 0$, we define
    the directional mollification of $f$, $\FF_{\ep} : \R^n \to \R$ as
    $\FF_{\ep}(x) :=
    F_{\ep}(x)+D_{\ep}(x)$, i.e.,
    \begin{align*}
        &\FF_{\ep}(x) :=
        \underbrace{(f * \varphi_{\ep})(x)}_{F_{\ep}(x)} +
        \underbrace{\int_{\R^n}f^{\circ}(s)(x-s)\varphi_{\ep}(x-s)ds}_{D_{\ep}(x)} = \int_{\R^n}(f(s)+f^{\circ}(s)(x-s))\varphi_{\ep}(x-s)ds,
    \end{align*}
    which is the sum of the mollification and the directional derivative term.
\end{definition}
Note that the previous definition makes sense because it is asked that the
function $f$, and its directional derivative are locally integrable functions.
Clearly, if $f : \R \to \R^n$ we just need to apply the previous definition to
each component of $f$.

We now present the theorem that will enable us, under some extra conditions
about $f$, to solve Problem \ref{prob:NewRegProblem} up to its
fourth statement. The fourth one is a particular application of the directional
mollification and will be solved in Section \ref{sec:ApplicationThreePointTwoSegments}.
The last point, regarding the computational cost,  can be seen to be relatively inexpensive, as
Definition \ref{def:NewMollification} involves straightforward integral calculations in a
compact support of $\R^n$ with locally integrable and smooth functions. Note that
the Theorem is presented for functions from $\R^n$ to $\R$. As previously mentioned, for a polygonal
chain, which is a function from $\R$ to $\R^n$, we just need to apply the results of the
theorem component-wise, i.e., for real-valued functions defined in $\R$. Therefore,
Theorem \ref{thm:TheoremThatSolvesNewRegularization} presents the solution
of Problem \ref{prob:NewRegProblem} to its full generality, allowing the use of the
directional mollification not only for curves in space, but also for surfaces.

\begin{theorem}\label{thm:TheoremThatSolvesNewRegularization}
    Let $f : \R^n \to \R$ satisfy the assumptions of
    Definition \ref{def:NewMollification}, $\varphi \in \DD_M(\R^n)$, and let
    $\ep > 0$. For the directional mollification of $f$, $\FF_{\ep} : \R^n \to
    \R$, the following statements are true.
    \begin{enumerate}[label=\arabic*., ref=\thetheorem.\arabic*]
        \item \label{thm:SolvReg1} If $f$ is differentiable almost everywhere, (that is,
        $\nabla f$ exists almost everywhere on $\R^n$), and
        $\partial_i f \in L_{\mathrm{loc}}^{q_i}(\R^n)$ for some $q_i \in [1,\infty]$
        for all $i \in \{1,\dots,n\}$, then $D_{\ep} \in C^{\infty}(\R^n)$
        and
        $D_{\ep} = \sum_{i=1}^{n}\partial_{i}f * \pi_i \varphi_{\ep}$
        where $x = (x_1,\dots,x_n) \in \R^n \mapsto \pi_i(x) = x_i$
        is the projection of the $i$-th component of a point in $\R^n$.
        Moreover, for any multi-index $\alpha \in \N_0^n$,
        $\partial^{\alpha}D_{\ep} = \sum_{i=1}^{n}\partial_{i} f *
        \partial^{\alpha}(\pi_i\varphi_{\ep})$, where
        $\partial^{\alpha}(\pi_i\varphi_{\ep}) =
        \sum_{\beta\leq\alpha}\partial^{\beta}
        \pi_i\partial^{\alpha-\beta}\varphi_{\ep}$.

        \item \label{thm:SolvReg2} If $f$ is locally Lipschitz, $\FF_{\ep} \to f$ as $\ep \to 0^+$ uniformly on compact subsets of $\R^n$.
        \item If $f$ is convex then
        $\FF_{\ep} \leq f$ and $D_{\ep} \leq 0$.
        \item If $f(x) = \dotProduct{a}{x} + c$ with $c \in \R$, $a \in
        \R^n$, and $\varphi$ is an even function,  then
        it holds that $\FF_{\ep} = f$ for all $\ep > 0$.
    \end{enumerate}
\end{theorem}
\begin{proof}
We prove each statement separately.
\begin{enumerate}
    \item  Since $F_{\ep} \in C^{\infty}(\R^n)$ we just need to
    prove that $D_{\ep} \in C^{\infty}(\R^n)$. Note that for any
    $x \in \R^n$ it holds that $D_{\ep}(x) =
    \int_{\R^n}\dotProduct{\nabla f(s)}{x-s}\varphi_{\ep}(x-s)ds$,
    because $\nabla f$ exists almost everywhere in $\R^n$. It is not difficult to
    check that  $D_{\ep}(x) =
    \int_{\R^n}\sum_{i=1}^{n}\partial_{i}f(s)\pi_{i}(x-s)\varphi_{\ep}(x-s)ds =
    \sum_{i=1}^{n}(\partial_{i}f * \pi_i \varphi_{\ep})(x)$, for all $x \in
    \R^n$. Since $\partial_{i}f \in L_{\mathrm{loc}}^{q_i}(\R^n)$ for all $i \in
    \{1,\dots,n\}$, we can use the results of \cite[Proposition
    4.20]{brezis2011functional} to conclude that $D_{\ep} \in C^{\infty}(\R^n)$,
    and that $\partial^{\alpha}(\partial_{i}f * \pi_i\varphi_{\ep}) =
    \partial_{i}f * \partial^{\alpha}(\pi_i\varphi_{\ep})$ since
    $\pi_i\varphi_{\ep} \in \DD(\R^n)$ for all $i \in \{1,\dots,n\}$ and all
    mollifier $\varphi$. The last equality comes from the Leibniz rule. This solves
    the first statement of Problem \ref{prob:NewRegProblem}.

    \item Assume $f$ is locally Lipschitz. Let $K$ be a compact subset of $\R^n$, and let $K_{2\ep} := \{x \in \R^n \mid \inf_{y\in K}\|x-y\| \leq 2\ep\}$, which is also compact. Since $K_{2\ep}$ is compact, $f|_{K_{2\ep}}$ is (globally) Lipschitz, with Lipschitz constant $L = L(K_{2\ep}) > 0$. Let $x, y \in K$ and $s \in \overline{B}(0,\ep)$. Observe that $\|x-y-s\| \leq \|x-y\|+\|s\|$, from which it follows that $\inf_{y\in K}\|x-y-s\| \leq \|s\| \leq \ep$, that is $(x-s) \in K_{2\ep}$. Moreover, for $0<t \leq 1$, we have that $\inf_{y\in K}\|(x-s)+ts - y\| \leq \inf_{y\in K}\|x - s - y\| + t\|s\| \leq (t+1)\ep \leq 2\ep$, that is
    $(x-s)+ts \in K_{2\ep}$. Therefore
    \begin{equation*}
        \frac{|f((x-s)+ts)-f(x-s)|}{t} \leq \frac{L\|ts\|}{t} \leq L\ep,
    \end{equation*}
    for all $t \in (0,1]$. Whence $|f^{\circ}(x-s)(s)| \leq L\|s\| \leq L\ep$ for all $x \in K$ and all $s \in \overline{B}(0,\ep)$. It follows that
    \begin{equation*}
        \sup_{x\in K}|D_{\ep}(x)| \leq \sup_{x\in K}\int_{\overline{B}(0,\ep)}|f^{\circ}(x-s)(s)|\varphi_{\ep}(s)ds \leq L\ep \sup_{x\in K}\int_{\overline{B}(0,\ep)}\varphi_{\ep}(s)ds = L\ep.
    \end{equation*}
    Therefore, $D_{\ep} \to 0$ as $\ep \to 0^{+}$ uniformly on compact subsets of $\R^n$, and since $f$ is continuous, by Theorem \ref{thm:OurReferenceTheorem} we have that $\FF_{\ep} \to 0$ as $\ep \to 0^+$ uniformly on compact subsets of $\R^n$.

    \item Let $x \in \R^n$. Using the fact that
    $\int_{\R^n}\varphi_{\ep} = 1$, then $\FF_{\ep}(x)-f(x) =
    \int_{\R^n}(f(s)+f^{\circ}(s)(x-s) -
    f(x))\varphi_{\ep}(x-s)ds$. However, by Theorem \ref{thm:ConvexityProperties}, $f(x) \geq f(s) +
    f^{\circ}(s)(x-s)$. Being $\varphi_{\ep} \geq 0$, $\FF_{\ep}(x)-f(x)
    \leq 0$, from which the first result follows. Moreover by
    Theorem \ref{thm:OurReferenceTheorem}, $f \leq F_{\ep}$, thus $f
    \geq \FF_{\ep} = F_{\ep} + D_{\ep} \geq f + D_{\ep}$ which implies
    $D_{\ep} \leq 0$. This solves the third statement of Problem \ref{prob:NewRegProblem}.
    \item Note that for any $x \in \R^n$,
    $F_{\ep}(x) = (f * \varphi_{\ep})(x) =
    \int_{\R^n}\dotProduct{a}{x-s}\varphi_{\ep}(s)ds + c
    = \dotProduct{a}{x} + c$,
    because $s \mapsto \dotProduct{a}{s}\varphi_{\ep}(s)$ is an
    odd function integrated over a ball centered at the origin.
    Moreover,
    since $f \in C^1(\R^n,\R)$ then $f^{\circ}(s)(x-s) = \dotProduct{\nabla
        f(s)}{x-s}
    = \dotProduct{a}{x-s}$ which leads to  $D_{\ep}(x)  = 0$,
    by the same conclusion as above. Thus, $\FF_{\ep} = F_{\ep} = f$.
\end{enumerate}
\end{proof}
We now compare conventional and directional mollification, highlighting the properties gained and those that may be lost in the process.

\subsection{Comparison with conventional mollification}

\begin{table}[t]
\centering
\small
\caption{A summary of the primary distinctions between the conventional and directional mollification approaches. It is always assumed that $\varphi \in \DD_M(\R^n)$.}
\label{tab:my-table}
\begin{tabularx}{\textwidth}{>{\fontseries{b}\selectfont\raggedright\arraybackslash}p{3.8cm}XX}

\toprule
Property / Case & \textbf{Conventional Mollification} & \textbf{Directional Mollification} \\
\midrule
Assumptions & $f$ locally integrable & $f$ locally integrable and directionally differentiable at all points with locally integrable directional derivative \\ \addlinespace

Equation & $F_{\ep}(x) = \int_{\R^n}f(s)\varphi_{\ep}(x-s)ds$ & $\FF_{\ep}(x) = \int_{\R^n}(f(s)+f^{\circ}(s)(x-s))\varphi_{\ep}(x-s)ds$ \\ \addlinespace

Regularity & $F_{\ep} \in C^{\infty}$, Theorem \ref{thm:PropertiesOfMollifying} & If $\nabla f$ exists a.e., $\FF_{\ep} \in C^{\infty}$, Theorem \ref{thm:TheoremThatSolvesNewRegularization} \\ \addlinespace

Convergence on compact subsets & $f$ continuous, Theorem \ref{thm:PropertiesOfMollifying} & $f$ locally Lipschitz, Theorem \ref{thm:TheoremThatSolvesNewRegularization} \\ \addlinespace

Affine $f$ and even mollifier & $F_{\ep} = f$ for all $\ep > 0$, Theorem \ref{thm:OurReferenceTheorem} & $\FF_{\ep} = f$ for all $\ep > 0$, Theorem \ref{thm:TheoremThatSolvesNewRegularization} \\ \addlinespace

Convex $f$ & $F_{\ep}$ is convex and $f \leq F_{\ep}$ for all $\ep > 0$, Theorem \ref{thm:OurReferenceTheorem} & $\FF_{\ep}$ is not necessarily convex for all $\ep > 0$, but $\FF_{\ep} \leq f$ for all $\ep > 0$, Theorem \ref{thm:TheoremThatSolvesNewRegularization} \\ \addlinespace

Locally convex $f$ at a point & $F_{\ep}$ is locally convex for sufficiently small $\ep > 0$, Theorem \ref{thm:OurReferenceTheorem} & $\FF_{\ep}$ is locally convex for sufficiently small $\ep > 0$, Theorem \ref{thm:TheoremThatSolvesNewRegularization} \\ \addlinespace

Monotonic $f$ & $F_{\ep}$ is monotonic for all $\ep > 0$, Theorem \ref{thm:OurReferenceTheorem} & $\FF_{\ep}$ is not monotonic in general. It is monotonic if $f' = c$ almost everywhere for some $c \in \R$, Section \ref{subsec:ConvexityNotAlwaysTrue} \\ \addlinespace

Length of a mollified curve & Decreases relative to the length of $f$, Theorem \ref{thm:OurReferenceTheorem} & Depends on the curve; strictly increases for polygonal chains, Proposition \ref{prop:LengthAllSegments} \\ \addlinespace

Enclosure of curves & $F_{\ep}$ is contained in the convex hull of the image of $f$ for all $\ep > 0$, Theorem \ref{thm:OurReferenceTheorem} & $\FF_{\ep}$ is not, in general, contained in the convex hull of the image of $f$, Section \ref{subsec:RemarkAboutLength} \\ \addlinespace

Computation & Convolution of $f$ with a mollifier. Complexity depends on the algorithm used & Convolution of $f + f^{\circ}$ with a mollifier. Same numerical complexity as conventional mollification \\ \addlinespace

Vertex preservation of polygonal chains & No, due to convex hull enclosure & Yes, Theorem \ref{thm:ConvergenceToLinesToSwitchPolygonalChain} \\ \addlinespace

Curvature upper bound formulas & Yes, \cite{gonzalezcalvin2025efficientgenerationsmoothpaths} & Yes, Proposition \ref{prop:CurvatureBoundingOutside} and Lemma \ref{lem:CurvatureApproachPractical} \\ \addlinespace

Flexibility & Varying $\ep > 0$ and the choice of mollifier. & Varying $\ep > 0$ and the choice of mollifier and using $F_{\ep}+\gamma D_{\ep}$ for $\gamma \in \R$ generalizes both $F_{\ep}$ and $\FF_{\ep}$, Section \ref{sec:Combination} \\
\bottomrule
\end{tabularx}
\end{table}

Before detailing the properties that are not preserved under the directional
mollification framework, we provide a systematic overview of the fundamental
differences between the conventional and directional mollification approaches.
This comprehensive side-by-side contrast is presented in Table
\ref{tab:my-table}. For completeness, the table also integrates the specific
geometric properties concerning polygonal chains---which are rigorously
evaluated in Sections \ref{sec:ApplicationThreePointTwoSegments} and
\ref{sec:Combination}---alongside their corresponding textual cross-references.
As shown in the \textit{computation} row, directional mollification maintains
the same numerical complexity as the conventional approach. This equivalence
arises because, numerically, summation scales linearly, whereas the numerical
complexity of a convolution is---while dependent on the algorithm---higher than
linear. However, this condition holds only if the directional mollification is
computationally evaluated as the convolution of a sum rather than as two
separate convolutions. In the latter case, the total computational time would
roughly double compared to the conventional approach.
\subsubsection{Convexity and monotonicity are not generally
preserved}\label{subsec:ConvexityNotAlwaysTrue}
 Theorem \ref{thm:OurReferenceTheorem}
presents a result about monotonicity that does not always hold for the
directional mollification.
\begin{figure}[pos=t]
    \centering
    \includegraphics[width=\linewidth]{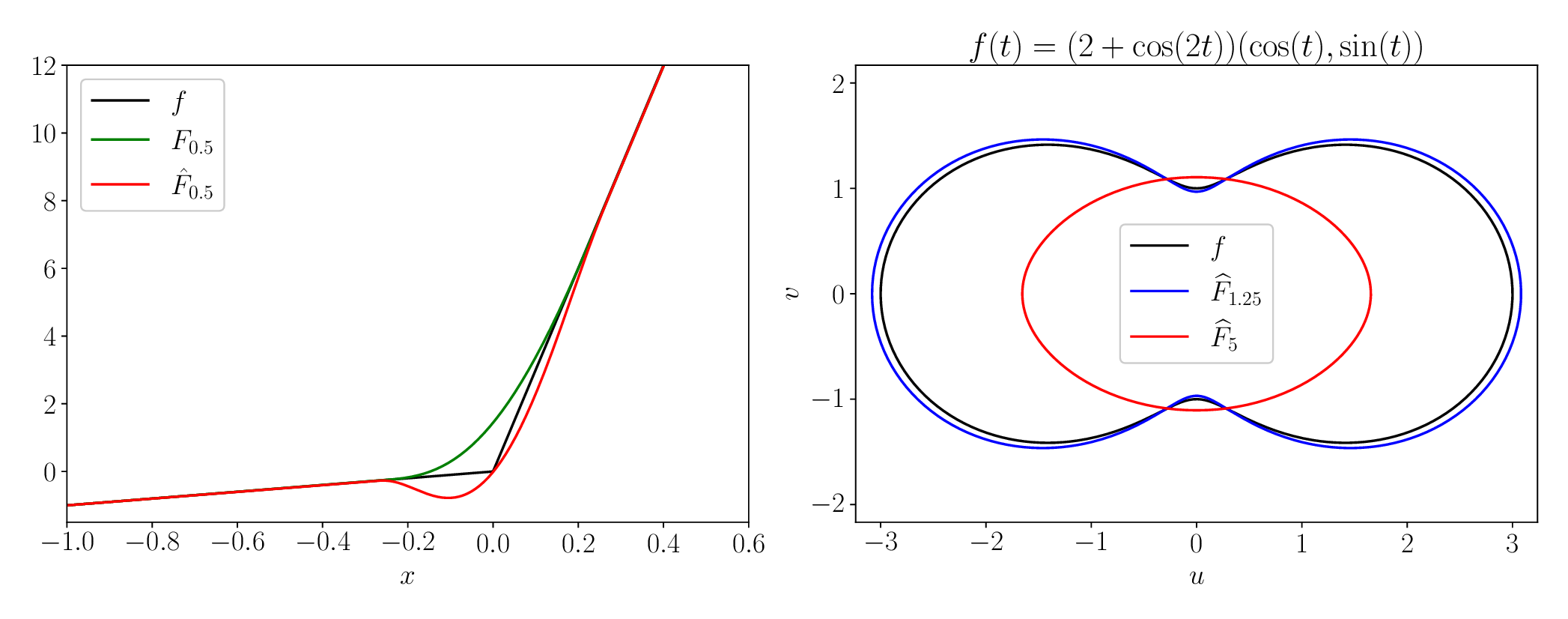}
    \caption{Representation of the conventional and directional mollification of
        the strictly increasing function $f(x)= x\ind_{(-\infty,0)}(x) +
        30x\ind_{[0,\infty)}(x)$ (left picture) and the function $t \in [0,2\pi] \mapsto
        (2+\cos(2t))(\cos(t),\sin(t)) \in \R^2$ (right picture, where the coordinates of $\R^2$ are shown as $(u,v)$), where in both cases the used
        mollifier $\varphi$ is
        as in \eqref{eq:OurMollifier}. For the latter, the convolution is carried out by extending its domain to $\R$. In the left picture, the black plot represents the original
        function $f$, while the green one is its conventional mollification, and the red
        one is its directional mollification, both with $\ep = 0.5$. The right picture
        represents the original function $f$ as a black plot, while the blue and red plots
        the directional mollifications $\FF_{\ep}$ for $\ep = 1.25$ and $\ep = 5$,
        respectively.}
    \label{fig:MonotonicityNotPreserved}
\end{figure}
It is in general false that if $f : \R \to \R$ is monotonic, then $\FF_{\ep}$ is
monotonic for all $\ep > 0$. See, for example, the left picture of Figure
\ref{fig:MonotonicityNotPreserved} (cf.  Theorem \ref{thm:OurReferenceTheorem}).
Nevertheless, a possible situation must be highlighted here. Suppose that the
desired curve consists of the sum of at most countably infinite shifted
Heaviside functions. Call this function $f : \R \to \R$. If $\lambda$ is the
Lebesgue measure, then $\lambda(\Z) = 0$, and in turn $f' = 0$ almost everywhere
on $\R$, which implies that $\FF_{\ep} = F_{\ep}$, i.e., the conventional
mollification and directional mollification coincide. This can be generalized
for a piecewise linear function with constant slope if $\varphi$ is even, $\supp
\varphi = [-1,1]$ and $\varphi \geq 0$. For if $f' = c$ almost everywhere, then
$D_{\ep}(x) = c\int_{\R}(x-s)\varphi_{\ep}(x-s)ds = c
\int_{[-\ep,\ep]}s\varphi_{\ep}(s)ds = 0$, because $s \mapsto s\varphi_{\ep}(s)$
is an odd function integrated over a symmetric interval. Therefore, if $f$ is
monotone increasing (resp.  decreasing) and $f' = c$ with $c \in \R$ almost
everywhere on $\R$, then $\FF_{\ep}$ is also monotone increasing (resp.
decreasing), because  by Theorem \ref{thm:OurReferenceTheorem}, $F_{\ep}$ is
monotone increasing (resp. decreasing). In addition, it is clear from Figure
\ref{fig:Figure_1_Example} or Figure \ref{fig:MonotonicityNotPreserved} that
even if $f$ is convex, it does not hold that $\FF_{\ep}$ is convex for all $\ep
> 0$.

\subsubsection{Length and enclosure of curves}\label{subsec:RemarkAboutLength}
The two final properties of Theorem \ref{thm:OurReferenceTheorem} regarding the
enclosure and length of $F_{\ep}$ deal with functions of the form $f:[a,b] \to
\R^n$.  It may seem from Figure \ref{fig:Figure_1_Example}, or Figure
\ref{fig:MonotonicityNotPreserved}  that, for any $\ep > 0$, $\LL(\FF_{\ep})
\geq \LL(f)$ and $f(\R) \subset \co(\FF_{\ep}(\R))$.  Nevertheless, this is not
always the case in the directional mollification. Consider as a counterexample
the (continuously differentiable) planar curve of the right picture in Figure
\ref{fig:MonotonicityNotPreserved}. Clearly, none of these properties hold for
$\ep = 5$, while they hold true for $\ep = 1.25$. Nevertheless, by the definition of the length of a curve, $\LL(\FF_{\ep}) \leq
\LL(F_{\ep}) + \LL(D_{\ep}) \leq \LL(f) + \LL(D_{\ep})$ for any $\ep > 0$.

\section{Directional mollification of polygonal chains}\label{sec:ApplicationThreePointTwoSegments}

Here we solve the fourth statement of Problem \ref{prob:NewRegProblem},
but before proceeding, we present a small lemma that will be used in this section.
Although its proof is elementary calculus, we still prove it to shed some light
on the notation. Moreover, we also present the formal equation of a polygonal
chain.
\begin{lemma}\label{lem:DirectionalDerivativeAndNormalDerivative}
    Let $f \in L^1_{\mathrm{loc}}(\R,\R^n)$, and suppose that the
    ordinary derivative $Df = (Df_1, Df_2, \dots, Df_n) : \R \to \R^n$
    exists almost everywhere. Then for all $s \in \R$ and almost all
    $x \in \R$ it holds that $f^{\circ}(x)(s) = s Df(x)$.
    Moreover, if $Df_i \in L^{q_i}_{\mathrm{loc}}(\R,\R^n)$, for some $q_i \in [1,\infty]$,
    for all $i \in \{1,\dots,n\}$, then for any $x \in \R$ and $n \in \N$,
    \begin{align*}
        D_{\ep}^{(n)}         &= n(Df * \varphi_{\ep}^{(n-1)}) + Df * \id
        \varphi_{\ep}^{(n)}.
    \end{align*}
\end{lemma}

\begin{proof}
    Let $g : \R \to \R^n$ be $g(t) = f(x+ts)$ with $s \in \R$ and $x\in \R$ such
    that $Df(x)$ exists. By the chain rule $g'(t) = Df(x+ts)s$ almost everywhere
    in $\R$. In particular, exists for $t=0$ and takes the value $g'(0) =
    Df(x)s$. Therefore $g(t)-g(0) = f(x+ts) - f(x)$, and so $Df(x)s = g'(0) =
    \lim_{t\to 0^+}\frac{f(x+ts)-f(x)}{t} = f^{\circ}(x)(s)$, almost everywhere
    $x \in \R$. The integration results follow from  Theorem
    \ref{thm:TheoremThatSolvesNewRegularization} and induction applied to each
    component of $f$.
\end{proof}

Although polygonal chains were introduced in the notation section, let us define
them explicitly, since we will make use of their equation from now on.
\begin{definition}\label{def:polygonalchain}
    Let $p \in \N$ with $p \geq 2$ and let $(P_0, \dots, P_p) \subset \R^n
    \times \cdots \times \R^n$ be a set of ordered points called knots or
    vertices. We define the $p+1$ knots/vertices polygonal chain in $\R^n$ to be
    the function $f: [0,p] \to \R^n$, expressed as
    \begin{equation}\label{eq:polygonalchain}
        f(t) := P_{r-1} + (P_{r} - P_{r-1})(t-(r-1)),
    \end{equation}
    where  $r \in \{1,\dots,p\}$ and  $t \in [r-1,r]$.
    To be able to convolve the previous function,
    we define its extension
    $\bar{f} : \R \to \R^n$ as
    \begin{equation}\label{eq:polygonalchainextended}
        \bar{f}(t) := \begin{cases}
            P_0 + (P_1-P_0)t, & t \leq 1 \\
            P_{r-1} + (P_{r} - P_{r-1})(t-(r-1)), & r \in \{2,\dots,p\}, t \in [r-1,r] \\
            P_{p-1} + (P_{p}-P_{p-1})(t-(p-1)), & t \geq p
        \end{cases}.
    \end{equation}
    Note that this function is locally Lipschitz, hence locally integrable and
    differentiable almost everywhere with locally integrable derivative. Thus,
    the requirements of Theorem \ref{thm:TheoremThatSolvesNewRegularization}
    hold.
\end{definition}
A seven-vertex polygonal chain was shown in Figure \ref{fig:Figure_1_Example},
as well as its conventional and directional mollification. As stated in the
definition of the polygonal chain, the assumptions of Theorem
\ref{thm:TheoremThatSolvesNewRegularization} and Lemma
\ref{lem:DirectionalDerivativeAndNormalDerivative} hold automatically.  Thus, we
can compute the directionally mollified curve as $ \FF_{\ep} = (\bar{f} *
\varphi_{\ep}) + (\id \varphi_{\ep} * D\bar{f}) = (\bar{f}_i *
\varphi_{\ep})_{i=1}^{n} + (D\bar{f}_i * \id \varphi_{\ep} )_{i=1}^{n}$.
Moreover, note that the assumptions of Theorem \ref{thm:SolvReg1} and
\ref{thm:SolvReg2} hold for \eqref{eq:polygonalchainextended}, which implies the
directionally mollified curve converges both pointwise and uniformly on compact
subsets of $\R$ to the curve. We will always make use of the extended definition
of \eqref{eq:polygonalchain} since it is needed to carry out the convolution
properly. Nevertheless, we will always show the curves in their original domain
of definition, which is the one of interest.

Since we are going to exploit the properties of a mollifier's evenness, we now define the following.
\begin{definition}
    We define $\DD_{E}(\R^n) := \{\varphi \in \DD_M(\R^n) \mid \varphi \text{ is an even function}\}$, i.e., the set of even nonnegative mollifiers whose support is $\overline{B}(0,1)$.
\end{definition}
\subsection{Vertex preservation of polygonal chains under directional mollification}\label{subsec:VertexPreservation}
We now prove that the directional mollification keeps the vertices of
the polygonal chain invariant, and we also compute explicit neighborhoods in which
the polygonal chain is equal to the conventional and directional mollification
approaches. See Figure \ref{fig:Mollification} for a representation of these claims and
the sets computed in the following theorem.
\begin{figure}[pos=t]
    \centering
    \includegraphics[width=\linewidth]{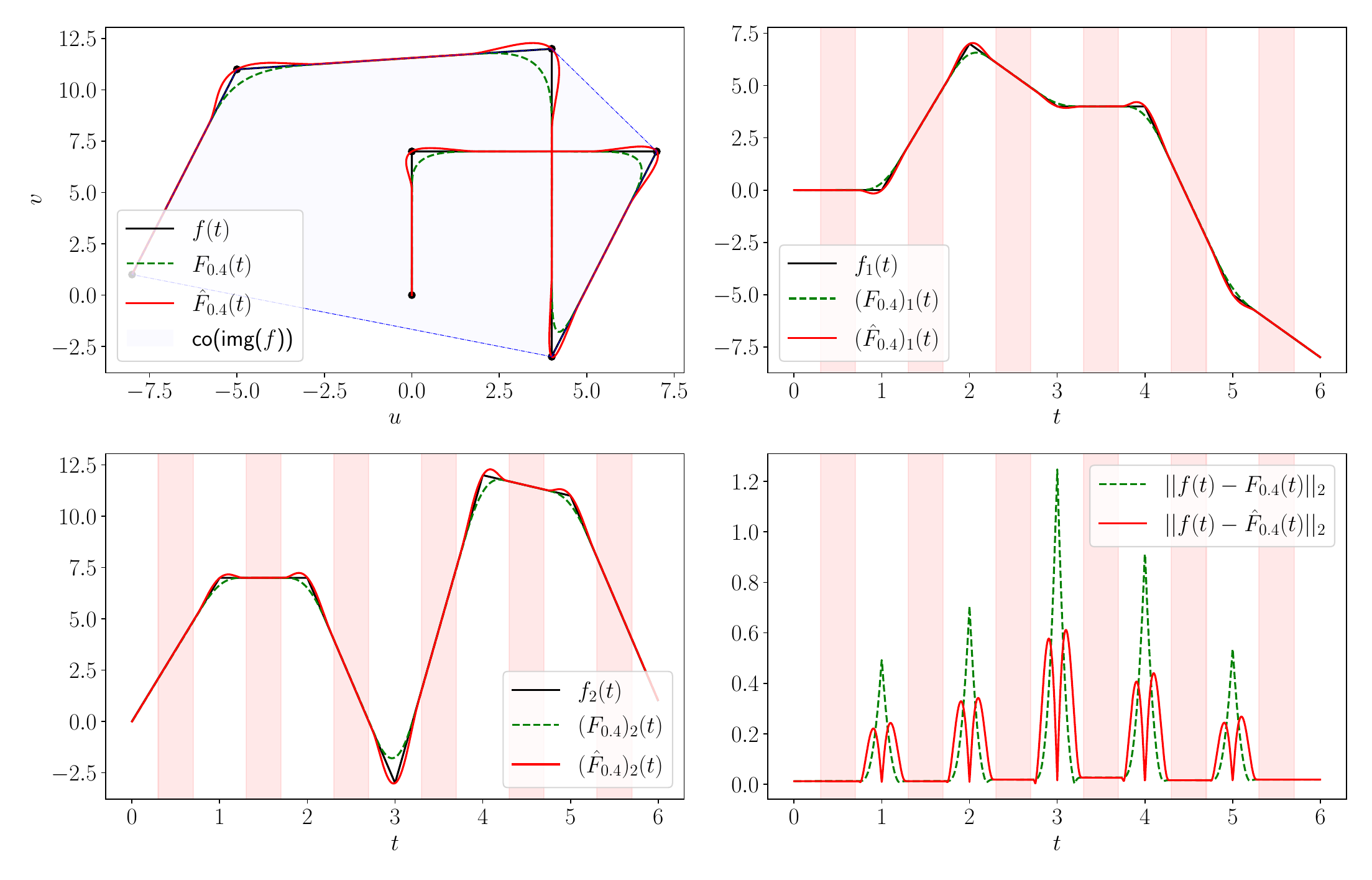}
    \caption{Representation of the conventional and directional mollification
        approaches for a $p=6$ polygonal chain in $\R^2$ defined as in \eqref{eq:polygonalchain},
        where the coordinates of $\R^2$ are expressed as $(u,v)$.
        The top left picture represents in black the original polygonal chain curve $f$,
        in dashed green its conventional mollification $F_{0.4}$, in red its
        directional mollification $\FF_{0.4}$, and as a blue transparent patch the convex
        hull of the vertices. In all cases, the mollifier used is the one presented
        in \eqref{eq:OurMollifier}. The top right and bottom left pictures represent, respectively,
        the first and second components of the polygonal chain, as well as its conventional
        and directional mollifications with the same color codes as in the top-left picture. The
        solid red regions represent the sets $V_{r}$ of Theorem \ref{thm:ConvergenceToLinesToSwitchPolygonalChain} for $r \in \{1,\dots,p\}$.
        The bottom right picture represents the errors (computed using the Euclidean norm)
        between the original function and the conventional mollification in dashed green,
        and the directional mollification in red. The solid red regions represent,
        as in the top-right and bottom-left pictures, the sets $V_r$.}
    \label{fig:Mollification}
\end{figure}
\begin{theorem}\label{thm:ConvergenceToLinesToSwitchPolygonalChain}
    Let $f$ be a polygonal chain as in Definition
    \ref{def:polygonalchain} which has already been extended
    (that is $f = \bar{f}$ in Definition \ref{def:polygonalchain}). Let
    $\varphi \in \DD_E(\R)$. The following statements hold.
    \begin{itemize}
        \item If $\ep \in (0,1)$ it holds that $\FF_{\ep}(k) = f(k)$
        for $k \in \{0,1,\dots,p\}$.
        \item Let $\ep \in (0,\frac{1}{2})$, $r \in \{1,2,\dots,p\}$ and define $V_r:= [r-1+\ep,r-\ep]$.
        Then it holds that $ f|_{V_r} = F_{\ep}|_{V_r} = \FF_{\ep}|_{V_r}$.
    \end{itemize}
\end{theorem}
\begin{proof}
    For the first statement, let $\ep \in (0,1)$ and
    $k \in \{1,2,\dots,p-1\}$. Then, by
    noting $\int_{(k-\ep,k+\ep)}\varphi_{\ep}(k-s)ds = 1$,
    \begin{align*}
        \FF_{\ep}(k) &=
        \int_{[k-\ep,k+\ep]}(f(s)+(k-s)Df(s))\varphi_{\ep}(k-s)ds \\
        &= \int_{[k-\ep,k]}(P_{k-1}+(P_{k}-P_{k-1})(s-(k-1)) + (k-s)(P_{k}-P_{k-1}))\varphi_{\ep}(k-s)ds \\
        &+\int_{[k,k+\ep]}(P_{k}+(P_{k+1}-P_{k})(s-k) + (k-s)(P_{k+1}-P_{k}))\varphi_{\ep}(k-s)ds \\
        &= P_{k}\left(\int_{k-\ep}^{k}\varphi_{\ep}(k-s)ds +
                     \int_{k}^{k+\ep}\varphi_{\ep}(k-s)ds\right) \\
        &= P_k \int_{[k-\ep,k+\ep]}\varphi_{\ep}(k-s)ds = P_k = f(k).
    \end{align*}
    The results for $k \in \{0,p\}$ follow in the same manner because the
    function has been extended.

    For the second statement, let the assumptions hold. Fix $r \in
    \{1,\dots,p\}$ and consider $V_r$. Take $t \in V_r$ and note $F_{\ep}(t) =
    \int_{[-\ep,\ep]}f(t-s)\varphi_{\ep}(s)ds$ which implies that $r-1 \leq
    t-\ep \leq t-s \leq t+\ep  \leq r$, for all $ s \in [-\ep,\ep]$. Thus, for
    any $t \in V_r$ the function to be mollified is the line $f(t) =
    P_{r-1}+(P_{r}-P_{r-1})(t-(r-1))$ by construction. By Theorem
    \ref{thm:OurReferenceTheorem} and Theorem
    \ref{thm:TheoremThatSolvesNewRegularization} we obtain the desired result.
\end{proof}
\begin{remark}\label{rmk:VersatileApproach}
    The previous theorem has several applications. Indeed, if $\ep
    \in (0,1)$, it states that the original function and its conventional and
    directional mollifications coincide in all segments in a set centered in
    each of them. This implies that, outside neighborhoods containing the
    non-differentiable vertices, the conventional and directional mollifications
    coincide with the curve. Moreover,  since both mollifications are exactly
    equal in those sets, a switch between them can be made \textit{without any
    discontinuity in the curve nor any of its derivatives}. In fact, we obtain a
    $C^{\infty}$ curve, being the switch of two $C^{\infty}$ curves in a set
    where they are equal. See Figure \ref{fig:Mollification} for a
    representation of the sets $V_r$, as well as each component of $f$ and its
    conventional and directional mollifications.
\end{remark}

\subsection{Local convexity preservation and length increase under directional mollification for polygonal chains}\label{subsec:LocalConvexityLengthPolygonalChains}

We will now consider each component of polygonal chain functions in $\R^n$, $f =
(f_1,,\dots,f_n)$ as in \eqref{eq:polygonalchain} to obtain convexity-like
results. Recall from Theorem \ref{thm:OurReferenceTheorem} that, if the original
real-valued function is convex, then so is its conventional mollification. This
is not always the case in the directional mollification as discussed in Section
\ref{subsec:ConvexityNotAlwaysTrue}. Nevertheless, as shown below, since each
component of $f$ is either locally convex, concave, or both, the directional
mollification of any component of $f$ preserves its convexity (resp. concavity)
in a neighborhood of its local minimum (resp. maximum). Clearly, as shown in
Theorem \ref{thm:TheoremThatSolvesNewRegularization}, if one of the components
is a straight line, then so is its directional mollification.

Since $\FF_{\ep} = (f_i*\varphi_{\ep})_{i=1}^{n} + (Df_i * \id
\varphi_{\ep})_{i=1}^{n}$, we are going to restrict ourselves to any component
$f_i$ of $f$, for $i \in \{1,\dots,n\}$. It is clear that, if we restrict $f_i$
for some $i$ to the set $[0,2]$ it has the form
\begin{equation}\label{eq:RealLineCase}
    t \in [0,2] \mapsto f_i|_{[0,2]}(t) = \begin{cases}
        y_0 + (y_1-y_0)t, & 0 \leq t \leq 1 \\
        y_1 + (y_2-y_1)(t-1), & 1 \leq t \leq 2
    \end{cases},
\end{equation}
with $y_0, y_1,y _2 \in \R$. Observe there is no loss of generality if $f$ were
to be defined on the set $[a,b]$ with $-\infty < a < b < \infty$, with midpoint
$(a+b)/2$. The purpose to restrict the component to $[0,2]$, is due to the fact
that if $\ep \in (0,1)$, only two segments are considered in the convolution at
the vertices of the polygonal chain.
\begin{lemma}\label{lem:ConvexityIsPreservedNearTheMinimum}
    Let $f_i|_{[0,2]}$ be as in \eqref{eq:RealLineCase}, and define $g :=
    f_i|_{[0,2]}$ in order to ease the notation. Consider its extension
    $\bar{g}$ (as done in \eqref{eq:polygonalchainextended}). Suppose that $y_0
    > y_1$ and $y_1 < y_2$ (i.e., $g$ is convex). Let $\varphi \in \DD_E(\R)$.
    Then, the directional mollification of $\bar{g}$, denoted as $\GG_{\ep} =
    G_{\ep} + D_{\ep} = \bar{g} * \varphi_{\ep} + (D\bar{g} * \id
    \varphi_{\ep})$ for $\ep > 0$ satisfies the following properties:
    \begin{itemize}
        \item The directional derivative term satisfies
        $D_{\ep}'(1) = 0$ and $D_{\ep}''(1) > 0$, i.e., it has a minimum at
        $t = 1$.
        \item For each $\ep > 0$ exists a $\delta = \delta(\ep) > 0$,
        such that if $V:=(1-\delta,1+\delta)$ then $\GG_{\ep}|_{V}$ is
        convex and $\GG_{\ep}|_{V} \leq g|_{V}$.
    \end{itemize}
\end{lemma}
\begin{proof}
    We prove each statement separately.
    \begin{itemize}
        \item Firstly observe that $\bar{g}$ is continuous and
        differentiable almost everywhere,
        with derivative $D\bar{g}$ which is also locally integrable.
        Thus it holds by Theorem \ref{thm:ConvexityProperties}
        and Lemma \ref{lem:DirectionalDerivativeAndNormalDerivative} that
        $D_{\ep}' = D\bar{g} * (\id\varphi_{\ep})'$ and
        $D_{\ep}'' = D\bar{g} * (\id\varphi_{\ep})''$.
        We proceed with the second derivative.
        \begin{align*}
            D_{\ep}''(1) &=
            \int_{[-\ep,\ep]}D\bar{g}(1-s)(\id \varphi_{\ep})''(s)ds \\
            &=(y_1-y_0)\int_{[0,\ep]}(\id \varphi_{\ep})''(s)ds
            +
            (y_2-y_1)\int_{[-\ep,0]}(\id \varphi_{\ep})''(s)ds \\
            &=(y_1-y_0)(\varphi_{\ep} + \id \varphi_{\ep}')\Big|_{0}^{\ep}
            + (y_2-y_1)(\varphi_{\ep} + \id \varphi_{\ep}')\Big|_{-\ep}^{0}
            \\
            &=\varphi_{\ep}(0)(y_2+y_0-2y_1),
        \end{align*}
        which is strictly positive because $y_0 > y_1$, $y_2 > y_1$ and
        $\varphi \geq 0$.  In the same fashion we can show that $D_{\ep}'(1)
        = 0$, which implies that $t=1$ is a local minimum of $D_{\ep}$.
        \item By the previous statement, $D_{\ep}''(1) > 0$ for all $\ep >
        0$.  Fixed $\ep > 0$, we know that $D_{\ep} \in
        C^{\infty}(\R,\R)$ by
        Theorem \ref{thm:TheoremThatSolvesNewRegularization}. Therefore, there
        exists a $\eta = \eta(\ep) > 0$ such that $D_{\ep}''(t) > 0$
        for all $t \in (1-\eta,1+\eta)$. Take $\delta = \min\{1,\eta\}$ and
        define $V:= (1-\delta,1+\delta)$. This implies that
        $D_{\ep}|_{V}$ is convex, which in turn makes $\GG_{\ep}|_{V}$
        convex, since by Theorem \ref{thm:OurReferenceTheorem},
        $G_{\ep}|_{V}$ is also convex because $g$ is. By
        Theorem \ref{thm:TheoremThatSolvesNewRegularization} we also have that
        $\GG_{\ep}|_{V} \leq g|_{V}$ because $g|_{V}$ is locally convex.
    \end{itemize}
    Clearly, a natural analogous result holds if $g$ is concave. I.e., $t=1$ is a maximum
    for $D_{\ep}$ and $\GG_{\ep}$ is locally concave near $t=1$. Moreover, if $\ep \in (0,1)$,
    it is then immediate to extend the previous lemma to any component defined in $[0,p]$, since
    we need to take into account neighboring segments of each vertex.
\end{proof}

The preceding lemma is what justifies the shape of $\FF_{\ep}$ in Figure
\ref{fig:Mollification}, see also Figure \ref{fig:MonotonicityNotPreserved} for
a function with the same geometric properties as in \eqref{eq:RealLineCase} but
defined in $[-1,1]$ with midpoint $0$.

We are now interested in discussing the length of the directionally mollified
curve. It seems that the length of the latter is always larger than the former,
as seen in Figures \ref{fig:Figure_1_Example} and \ref{fig:Mollification}.
However, as discussed in Section \ref{subsec:RemarkAboutLength}, this is not
always the case, but it is true if we restrict ourselves to the polygonal chain
curve.

\begin{proposition}\label{prop:LengthAllSegments}
    Let $f$ be defined as in Definition \ref{def:polygonalchain} and $\bar{f}$
    be its extension. Let $\varphi \in \DD_E(\R)$, $\FF_{\ep} : \R \to \R^n$ be
    the directional mollification of $\bar{f}$ and $F_{\ep} : \R \to \R^n$ be
    the conventional mollification of $\bar{f}$. If $\ep \in (0,1)$ then
    $\LL(F_{\ep}|_{[0,p]}) \leq  \LL(f) \leq \LL(\FF_{\ep}|_{[0,p]}),$ where the
    length of the curve is computed with respect to the $\ell_2$ norm.
\end{proposition}
\begin{proof}
    Since $\ep \in (0,1)$ we know by Theorem
    \ref{thm:ConvergenceToLinesToSwitchPolygonalChain} that for any $k \in
    \{0,\dots,p\}$ it holds that $\FF_{\ep}(k) = f(k) = \bar{f}(k)$. It follows
    immediately that $\LL(\bar{f}|_{[i,i+1]}) \leq \LL(\FF_{\ep}|_{[i,i+1]})$
    for $i \in \{0, \dots, p-1\}$. Indeed, recall the fact that the geodesic
    between two points in $\R^n$ is a straight line. This proves that
    $\LL(\bar{f}|_{[0,p]}) = \sum_{i=1}^{p-1}\LL(\bar{f}|_{[i,i+1]}) \leq
    \sum_{i=1}^{p-1}\LL(\FF_{\ep}|_{[i,i+1]}) = \LL(\FF_{\ep}|_{[0,p]})$. The
    equality $ \sum_{i=1}^{p-1}\LL(g|_{[i,i+1]}) = \LL(g|_{[0,p]})$ for any
    continuous $g : [0,p] \to \R^n$ can be found in \cite[Proposition
    2.3.4]{burago2001course}. Finally, the fact that $\LL(F_{\ep}|_{[0,p]}) \leq
    \LL(f)$ comes directly from Theorem \ref{thm:OurReferenceTheorem}.
\end{proof}

\subsection{Curvature upper bounds for directionally mollified polygonal chains}
We proceed now to obtain an exact formula as well as an upper bound on the
curvature for the directionally mollified polygonal chain. Having an upper bound
on the curvature related to the parameter $\ep > 0$ plays a crucial role in
geometric design, path planning, industrial robotics or CNC machining. Indeed,
when some system limits the maximum curvature dynamics, by choosing an
appropriate $\ep >0$ we can ensure that the curve fits its dynamics, while still
having a smooth curve with all the aforementioned properties of Sections
\ref{sec:Mollifiers}, \ref{sec:NewMollification},
\ref{subsec:VertexPreservation} and
\ref{subsec:LocalConvexityLengthPolygonalChains}.
\subsubsection{The three vertices case}
We restrict ourselves first to the three-vertices case, for which the following proposition presents such an
analytical, easy-to-compute, upper bound.
\begin{proposition}\label{prop:CurvatureBoundingOutside}
    Let $f$ be the curve defined in
    Definition \ref{def:polygonalchain} with $p=2$, and let
    $\bar{f}$ be its extension. Let $\varphi \in \DD_E(\R)$, and
    let $\ep > 0$. Define $\bar{s} :=
    \frac{\dotProduct{\PP_2-\PP_1}{\PP_2}}{\|\PP_2-\PP_1\|_2^2}$, where $\PP_2 :=
    P_2-P_1$, $\PP_1 := P_1-P_0$ and suppose that
    \begin{equation*}
        0 \notin \{s\PP_1 + (1-s)\PP_2 \mid s \in \R\}.
    \end{equation*}
    We denote as $\FF_{\ep}$ the directional mollification of $\bar{f}$, and as $\kappa : \R \to \R$ its curvature, whose formula in absolute value reads
    \begin{equation}\label{eq:Formulakappa}
        t\in \R \mapsto |\kappa(t)| := \frac{\|\FF_{\ep}''(t)\wedge \FF_{\ep}'(t)\|_2}{\|\FF_{\ep}'(t)\|_2^3}, \quad \FF_{\ep}(t) \neq 0.
    \end{equation}
    The following statements are true:
    \begin{enumerate}
        \item The curve is regular in the sense that $\FF_{\ep}' \neq 0$. That is, the curvature $\kappa$ is well defined on $\R$. Moreover, it can be computed as
        \begin{align}\label{eq:CurvatureCompleteOutsideMollifier}
            &|\kappa(t)| = |2\varphi_{\ep}(t-1)+(t-1)\varphi_{\ep}'(t-1)|\frac{\|\PP_2 \wedge
                \PP_1\|_2}{L(t)}, \quad \forall t \in
            \R,
        \end{align}
        where, for any $t \in \R$,
        \begin{align*}
            L(t)&:=\|\PP_1(A_1(t)-(t-1)\varphi_{\ep}(t-1))
        +\PP_2(A_2(t)+(t-1)\varphi_{\ep}(t-1))\|_2^3, \\
            A_1(t)&: = \int_{(-\infty,1]}\varphi_{\ep}(t-s)ds, \\
            A_2(t)&:= \int_{[1,\infty)}\varphi_{\ep}(t-s)ds.
        \end{align*}
        \item In addition, an upper bound can be obtained as
    \begin{equation}\label{eq:CurvatureBoundedOutsideMollifier}
        \begin{split}
        |\kappa(t)| &\leq \left(\frac{2}{\ep}\|\varphi\|_{\infty}+\frac{1}{\ep^2}\|\varphi'\|_{\infty}\right)\frac{\|\PP_2\wedge
            \PP_1\|_2}{\|\bar{s}\PP_1+(1-\bar{s})\PP_2\|_2^3} \\
            &\leq \max\left\{\frac{1}{\ep},\frac{1}{\ep^2}\right\}\left(2\|\varphi\|_{\infty}+\|\varphi'\|_{\infty}\right)\frac{\|\PP_2\wedge\PP_1\|_2}{\|\bar{s}\PP_1
        + (1-\bar{s})\PP_2\|_2^3}\quad \forall
        t \in \R.
        \end{split}
    \end{equation}
    \end{enumerate}
\end{proposition}
\begin{proof}
We know $\FF_{\ep} = F_{\ep} + D_{\ep}$ with $F_{\ep} = \bar{f} * \varphi_{\ep}$
and $D_{\ep} = D\bar{f} * \id \varphi_{\ep}$ by Lemma
\ref{lem:DirectionalDerivativeAndNormalDerivative}. Moreover, it is shown in
\cite{gonzalezcalvin2025efficientgenerationsmoothpaths} that for any $t \in \R$,
$F_{\ep}'(t) = \PP_1A_1(t) + \PP_2A_2(t)$ and $F_{\ep}''(t) =
\varphi_{\ep}(t-1)(\PP_2-\PP_1)$. Since the assumptions of Lemma
\ref{lem:DirectionalDerivativeAndNormalDerivative} or Theorem
\ref{thm:TheoremThatSolvesNewRegularization} are satisfied by $\bar{f}$ and
$D\bar{f}$,
\begin{align*}
    D_{\ep}'(t) = (D\bar{f} * \varphi_{\ep})(t) + (D\bar{f} * \id
    \varphi_{\ep}')(t)
    &= \int_{\R}D\bar{f}(s)(\varphi_{\ep}(t-s) +
    (t-s)\varphi_{\ep}'(t-s))ds \\
    &= -\int_{\R}D\bar{f}(s)\frac{\partial}{\partial
        s}\left[(t-s)\varphi_{\ep}(t-s)\right]ds \\
    &= -\PP_1((t-1)\varphi_{\ep}(t-1)) +
    \PP_2((t-1)\varphi_{\ep}(t-1))  \\
    &= (\PP_2-\PP_1)(t-1)\varphi_{\ep}(t-1),
\end{align*}
where in the third equality we used the fact that $\partial_s
(x-s)\varphi_{\ep}(x-s) = -\varphi_{\ep}(x-s) - (x-s)\varphi_{\ep}'(x-s)$, and
in the fourth equality the fact that $\varphi_{\ep}$ has compact support.
Therefore $D_{\ep}'(t) = (t-1)\varphi_{\ep}(t-1)(\PP_2-\PP_1)$ and $D_{\ep}''(t)
= (\varphi_{\ep}(t-1) + (t-1)\varphi_{\ep}'(t-1))(\PP_2-\PP_1)$. Summing up the
terms, this implies that
\begin{equation}\label{eq:FormulaDerivativeSimple}
    \begin{split}
    \FF_{\ep}'(t) &= \PP_1(A_1(t) - (t-1)\varphi_{\ep}(t-1))
    +
    \PP_2(A_2(t) + (t-1)\varphi_{\ep}(t-1)), \\
    \FF_{\ep}''(t) &= (\PP_2-\PP_1)(2\varphi_{\ep}(t-1)+(t-1)\varphi_{\ep}'(t-1)).
    \end{split}
\end{equation}
Note that the real coefficients that multiply each $\PP_i$ sum to one in
\eqref{eq:FormulaDerivativeSimple} because $A_1(t)+A_2(t) =1$ for any $t \in
\R$,  which implies that $\FF_{\ep}'(t) \neq 0$ for all $t \in \R$ by
assumption. Moreover,
\begin{equation}\label{eq:FormulaWedge}
    \FF_{\ep}''(t) \wedge \FF_{\ep}'(t) =
    (2\varphi_{\ep}(t-1)+(t-1)\varphi_{\ep}'(t-1))(\PP_2\wedge\PP_1).
\end{equation}
Substituting \eqref{eq:FormulaDerivativeSimple} and
\eqref{eq:FormulaWedge} into \eqref{eq:Formulakappa}
results in \eqref{eq:CurvatureCompleteOutsideMollifier}.

We now proceed to compute the upper bound. First observe that
\begin{equation}\label{eq:ConvexOptProblem}
    \|\FF_{\ep}'(t)\|_2 \geq \min_{s \in \R}\|s\PP_1 +
    (1-s)\PP_2\|_2
    =: \min_{s\in \R}g(s),
\end{equation}
and the fact that $\min_{s\in \R}g(s) > 0$ by assumption.
Since $g : \R \to \R$ is a continuously differentiable convex function
it is necessary and sufficient that $g'(\bar{s}) = 0$ for $\bar{s}$ to
be a global minimum. Noting that
\begin{equation*}
    g'(\bar{s})=\frac{2\dotProduct{\PP_1-\PP_2}{\bar{s}\PP_1+(1-\bar{s})\PP_2}}{\|\bar{s}\PP_1+(1-\bar{s})\PP_2\|_2}
    = 0 \iff \bar{s} = \frac{\dotProduct{\PP_2-\PP_1}{\PP_2}}{\|\PP_2-\PP_1\|_2^2},
\end{equation*}
proves that $\bar{s}$ is a global minimum.
Since the norm is a positive number,
and $|t-1\|\varphi_{\ep}'(t-1)| =
\frac{|t-1|}{\ep^2}|\varphi'((t-1)/\ep)| \leq
\frac{1}{\ep^2}\|\varphi'\|_{\infty}$, then
\begin{equation*}
    |\kappa(t)| \leq
    \left(\frac{2}{\ep}\|\varphi\|_{\infty}+\frac{1}{\ep^2}\|\varphi'\|_{\infty}\right)\frac{\|\PP_2\wedge\PP_1\|_2}{\|\bar{s}\PP_1
        + (1-\bar{s})\PP_2\|_2^3}, \quad \forall t \in \R.
\end{equation*}
This proves \eqref{eq:CurvatureBoundedOutsideMollifier}, and finishes
the proof.
\end{proof}

\begin{figure}
    \centering
    \includegraphics[width=\linewidth]{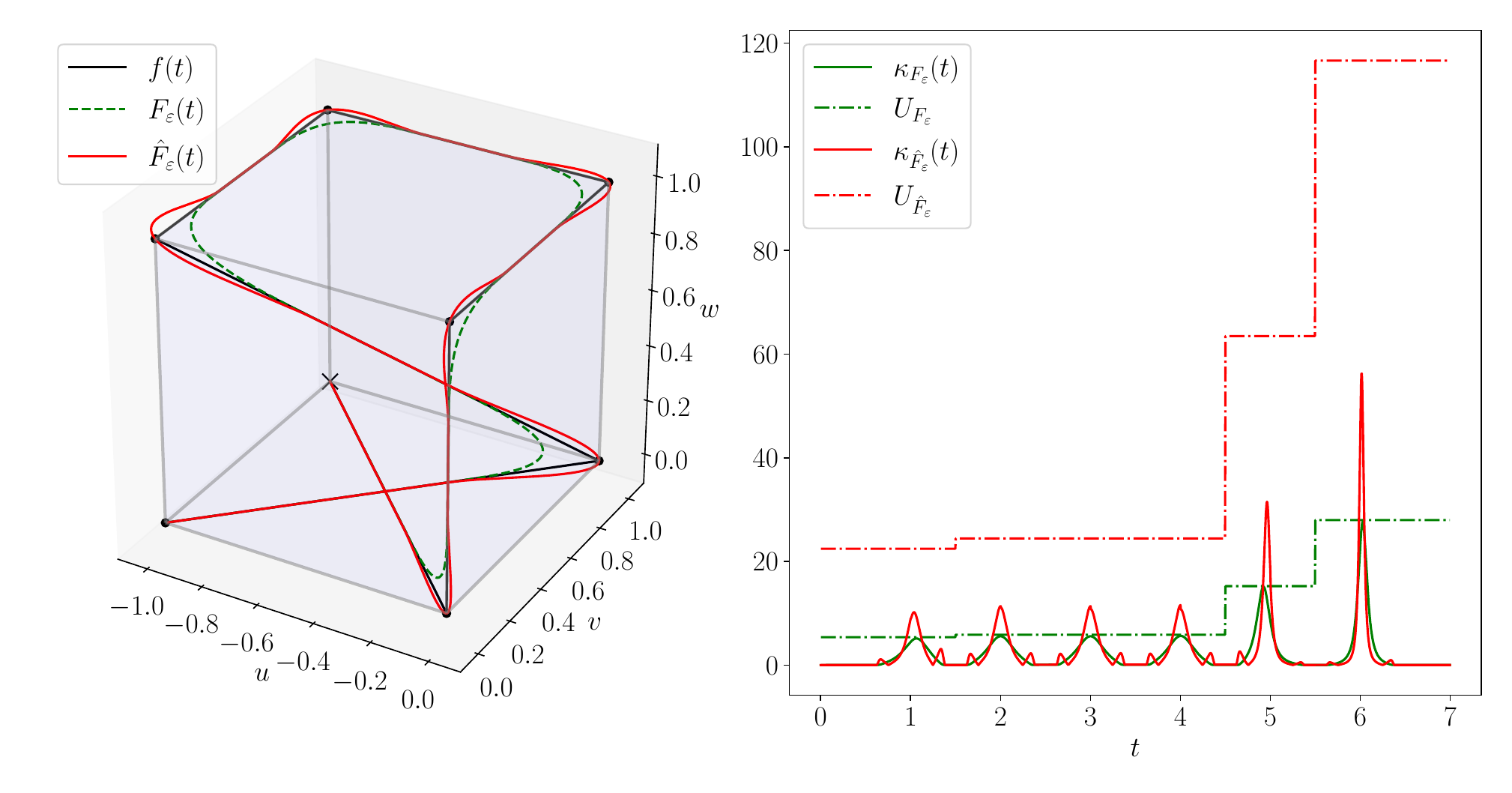}
    \caption{Representation of a $p=7$ polygonal chain in $\R^3$, and the curvature
    and curvature bounds of the conventional and directional mollifications. The
    coordinates in $\R^3$ are expressed as $(u,v,w)$.
    The left picture represents in black the polygonal chain curve $f$ confined
    in a unit volume cube in $\R^3$, shown in blue. It also represents its
    conventional $F_{\ep}$ and directional mollifications $\FF_{\ep}$ in dashed
    green and solid red, respectively. In all mollifications, the used mollifier is the one
    presented in \eqref{eq:OurMollifier}, with $\ep = 0.4$. Right picture
    represents in green the curvature of the conventional mollified curve
    $F_{\ep}$, and as a dashed green plot, the curvature upper bounds
    $\kappa_{F_{\ep}}$ for each pair of segment as follows. In the set $[0,1.5]$
    the curvature bound is computed using \eqref{eq:UpperBoundConvMolli} for the
    polygonal chain consisting of the points $P_0,P_1$ and $P_2$. In the set
    $[k+0.5,k+1.5]$ the curvature bound is computed using
    \eqref{eq:UpperBoundConvMolli} for the polygonal chain consisting of the
    points $P_{k}, P_{k+1}$ and $P_{k+2}$ for $k \in \{1,\dots,p-1\}$. The solid
    and red dashed plots represent, respectively, the curvature of
    $\FF_{\ep}$, and its curvature upper bounds $\kappa_{\FF_{\ep}}$ computed in the
    same manner as for the conventional mollification, but using the formula
    \eqref{eq:CurvatureBoundedOutsideMollifier} obtained in Proposition
    \ref{prop:CurvatureBoundingOutside}.}
    \label{fig:RepresentationCurvatureComparison}
\end{figure}
\begin{remark}
    Note that the curvature of the conventional mollifier,
    that is, the curvature of the curve $F_{\ep}$ (see \cite[Equation 7]{gonzalezcalvin2025efficientgenerationsmoothpaths}) denoted as
    $\kappa_{F_{\ep}}$, can be expressed as
    \begin{equation*}
        |\kappa_{F_{\ep}}(t)| = |\varphi_{\ep}(t-1)|\frac{
    \|\tilde{P}_2\wedge \tilde{P}_1\|_2}{\|\tilde{P}_1A_1(t)+\tilde{P}_2A_2(t)\|_2^3},
    \end{equation*}
    as long as $F_{\ep}' \neq 0$.
    It also admits an upper bound:
    \begin{equation}\label{eq:UpperBoundConvMolli}
        |\kappa_{F_{\ep}}(t)| \leq \frac{1}{\ep}\|\varphi\|_{\infty}\|\PP_2 \wedge \PP_1\|_2M(\PP_1,\PP_2),
    \end{equation}
    where
    \begin{equation*}
    M(\tilde{P}_1,\tilde{P}_2)
    := \begin{cases}
        \frac{1}{\left|\left|\tilde{P}_1
    \bar{s} +
    \left(1-\bar{s}\right)\tilde{P}_2\right|\right|_2^3},
    & 0 \leq \bar{s} \leq 1 \\
    \max\left\{\frac{1}{\|\tilde{P}_1\|_2^3}, \frac{1}{\|\tilde{P}_2\|_2^3}\right\}, & \text{ otherwise }
    \end{cases}.
    \end{equation*}

    Denote as $\kappa_{\FF_{\ep}}$ the curvature of $\FF_{\ep}$ and observe
    that, at $t = 1$, $2\kappa_{F_{\ep}}(1) = \kappa_{\FF_{\ep}}(1)$. Moreover,
    in the directional mollification case, the set in which the convex
    optimization problem (cf. equation \eqref{eq:ConvexOptProblem}) is carried
    out is in $\R$, while for the conventional mollifier it
    is in the set $[0,1]$,
    which is the justification of $M(\PP_1,\PP_2)$. Clearly, $\min_{s\in
    \R}\|s\PP_1+(1-s)\PP_2\|_2 \leq \min_{s\in[0,1]}\|s\PP_1+(1-s)\PP_2\|_2$.
    Moreover, in \eqref{eq:CurvatureBoundedOutsideMollifier} there is a
    dependence with $\frac{1}{\ep^2}$, $\varphi$ and $\varphi'$ while for the
    conventional mollifier is $\frac{1}{\ep}$, and $\varphi$. Thus, the upper
    bound is greater in the case of the directional mollification than for the standard mollifier \eqref{eq:OurMollifier}. See Figure
    \ref{fig:RepresentationCurvatureComparison} for a comparison between the
    conventional mollification and directional mollification, as well as their
    curvatures and upper bounds, computed for each pair of segments for a
    three-dimensional curve.
\end{remark}

\subsubsection{The general case}
It is then  natural to ask for an upper bound on the curvature regarding an
arbitrary number of vertices. From a practical perspective, to establish an
analytical upper bound on the curvature in this scenario, we propose the
following approach. We present this strategy rigorously as a lemma, whose result
follows directly from Theorem
\ref{thm:ConvergenceToLinesToSwitchPolygonalChain}.
\begin{lemma}\label{lem:CurvatureApproachPractical}
    Let $f$ be the curve defined in
    Definition \ref{def:polygonalchain}, and let
    $\bar{f}$ be its extension. Let $\varphi \in \DD_E(\R)$ and $\ep\in (0,\frac{1}{2})$.
    Define
    \begin{align*}
        \PP_r &:= P_{r}-P_{r-1}, \quad r \in \{1,\dots,p\}, \\
        \bar{s}_r &:= \frac{\dotProduct{\PP_{r+1}-\PP_{r}}{\PP_{r+1}}}{\|\PP_{r+1}-\PP_{r}\|_2^2}, \quad r \in \{1,\dots,p-1\}, \\
        A_r(t) &:= \int_{[r-1,r]}\varphi_{\ep}(t-s)ds, \quad r \in \{2,\dots,p-1\}, \quad t \in \R, \\
        A_{1}(t)&:= \int_{(-\infty,1]}\varphi_{\ep}(t-s)ds, \quad A_p(t):=\int_{[p,\infty)}\varphi_{\ep}(t-s)ds, \quad t \in \R.
    \end{align*}
    and suppose
    \begin{align*}
        0 \notin \{s\PP_{r} + (1-s)\PP_{r+1} \mid s \in \R\}, \quad r \in \{1,\dots,p-1\}.
    \end{align*}
    If we denote as $\FF_{\ep}$ the directional mollification of $\bar{f}$,
    the following statements are true:
    \begin{enumerate}
        \item The curvature of $\FF_{\ep}$ is well defined on $\R$ and can be computed
        for each $t \in \R$ as
        \begin{align}\label{eq:CurvatureCompleteOutsideMollifierFullSegments}
        &|\kappa(t)| = \begin{cases}
        |2\varphi_{\ep}(t-r)+(t-r)\varphi_{\ep}'(t-r)|\frac{\|\PP_{r+1} \wedge
            \PP_{r}\|_2}{L_r(t)},  &  t \in [r-\ep,r+\ep], r \in \{1,\dots,p-1\}\\
            0, & t \in\begin{cases}  [r-1+\ep,r-\ep] , r \in \{1,\dots,p\} \\
             (-\infty,\ep) \cup (p+\ep,\infty)\end{cases} \\
        \end{cases},
        \end{align}
        where
        \begin{align*}
            L_r(t)&:=\|\PP_r(A_r(t)-(t-r)\varphi_{\ep}(t-r))
        +\PP_{r+1}(A_{r+1}(t)+(t-r)\varphi_{\ep}(t-r))\|_2^3, \quad r \in \{1,\dots,p-1\}.
        \end{align*}
        \item     Furthermore, for each $r \in \{1,\dots,p\}$ an upper bound on the curvature of the trajectory generated by the consecutive vertices $P_{r-1}, P_{r},$ and $P_{r+1}$, denoted $\kappa_r$, can be obtained as
        \begin{equation}\label{eq:UpperBoundForPairOFSegments}
            \begin{split}
            |\kappa_{r}(t)| &\leq \left(\frac{2}{\ep}\|\varphi\|_{\infty}+\frac{1}{\ep^2}\|\varphi'\|_{\infty}\right)\frac{\|\PP_{r+1}\wedge
                \PP_r\|_2}{\|\bar{s}_r\PP_r+(1-\bar{s}_r)\PP_{r+1}\|_2^3} \\
                &\leq  \max\left\{\frac{1}{\ep},\frac{1}{\ep^2}\right\}\left(2\|\varphi\|_{\infty}+\|\varphi'\|_{\infty}\right)\frac{\|\PP_{r+1}\wedge
                \PP_r\|_2}{\|\bar{s}_r\PP_r+(1-\bar{s}_r)\PP_{r+1}\|_2^3}, \quad r \in \{1,\dots,p-1\}, \forall
            t \in \R,
            \end{split}
        \end{equation}
        and an upper bound for the complete curve as
        \begin{equation}\label{eq:CurvatureBoundedOutsideMollifierFullSegments}
            \begin{split}
            |\kappa(t)| &\leq \max_{r\in \{1,\dots,p-1\}}  \left(\frac{2}{\ep}\|\varphi\|_{\infty}+\frac{1}{\ep^2}\|\varphi'\|_{\infty}\right)\frac{\|\PP_{r+1}\wedge
                \PP_r\|_2}{\|\bar{s}_r\PP_r+(1-\bar{s}_r)\PP_{r+1}\|_2^3} \\
                &\leq \max_{r\in \{1,\dots,p-1\}}   \max\left\{\frac{1}{\ep},\frac{1}{\ep^2}\right\}\left(2\|\varphi\|_{\infty}+\|\varphi'\|_{\infty}\right)\frac{\|\PP_{r+1}\wedge
                \PP_r\|_2}{\|\bar{s}_r\PP_r+(1-\bar{s}_r)\PP_{r+1}\|_2^3},
                \quad \forall
            t \in \R.
            \end{split}
        \end{equation}
    \end{enumerate}
\end{lemma}
\begin{proof}
    Being $\ep \in (0,\frac{1}{2})$, we know, by Theorem
    \ref{thm:ConvergenceToLinesToSwitchPolygonalChain} that
    $f|_{V_r}=\FF_{\ep}|_{V_r}$ where $V_r = [r-1+\ep,r-\ep]$ for $r \in
    \{1,\dots,p\}$. Thus, $\kappa|_{V_r} = 0$ for all $r \in \{1,\dots,p\}$
    being $f|_{V_r}$ a straight line. The same result holds for $t \in
    (-\infty,\ep) \cup (p+\ep,\infty)$, because we are considering the extension
    of the curve. Moreover, in the sets $W_r := [r-\ep,r+\ep]$ for $r \in
    \{1,\dots,p-1\}$, we obtain the curvature expressions and curvature upper
    bounds using the results of Proposition \ref{prop:CurvatureBoundingOutside}.
    Indeed, due to the choice of $\ep > 0$ and the compactness of the support of
    the mollifier, only two adjacent segments are evaluated for any $t \in W_r$.
    From Proposition \ref{prop:CurvatureBoundingOutside} we get
    \eqref{eq:CurvatureCompleteOutsideMollifierFullSegments} and combining these
    results we obtain \eqref{eq:UpperBoundForPairOFSegments} and the upper bound
    \eqref{eq:CurvatureBoundedOutsideMollifierFullSegments}.
\end{proof}
\begin{remark}[A practical approach to the choice of $\ep > 0$]\label{rmk:PracticalApproach}
    We propose here two approaches to computing a valid choice of $\ep > 0$ to meet an upper-bound criterion.
    \begin{itemize}
        \item Given a polygonal chain with $p+1$ vertices and a pre-specified
        maximum allowable curvature, Lemma \ref{lem:CurvatureApproachPractical}
        and Equation \eqref{eq:UpperBoundForPairOFSegments} can be used to
        compute $p$ local parameters $\varepsilon_r$, one for each pair of
        segments. If $\varepsilon_r \in (0,\frac{1}{2})$ for all $r \in
        \{1,\dots,p\}$, the smoothing regions do not overlap. This lack of
        segment interaction guarantees the validity of the localized bound in
        \eqref{eq:UpperBoundForPairOFSegments}. Setting the global parameter
        $\varepsilon > 0$ as the maximum of all local $\varepsilon_r$ ensures
        that every segment pair satisfies the curvature criteria. Consequently,
        by virtue of \eqref{eq:CurvatureBoundedOutsideMollifierFullSegments},
        the curvature of the complete composite curve is guaranteed to satisfy
        the global constraints.

        \item Alternatively, if the condition $\varepsilon_r \in
        (0,\frac{1}{2})$ fails to hold for some $r \in \{1,\dots,p\}$, the value
        $\max_{r} \varepsilon_r$ can instead be used as an initial guess for an
        optimization problem. In this formulation, the decision variable is the
        global directional mollification parameter $\varepsilon$.
    \end{itemize}
In Figure \ref{fig:RepresentationCurvatureComparison}, since $\ep = 0.4$, the
upper bound for each pair of segments is computed as in Lemma
\ref{lem:CurvatureApproachPractical}. As can be seen, the curvature bound
estimates can indeed be used as the initial condition for an optimization
problem, or in this case, as an actual upper bound for the complete trajectory.
\end{remark}

\section{Combining the conventional and directional
mollifications}\label{sec:Combination}
This section focuses on the combination of the conventional and the directional
mollification approaches. In particular, given a  locally integrable
function\footnote{Clearly, this could be generalized to any function $f \in
L^p_{\mathrm{loc}}(\R^n,\R^m)$.} $f : \R \to \R^n$, we consider for any $\gamma \in \R$,
mollifier $\varphi \in \DD_E(\R)$ and $\ep > 0$, the new function $G_{\ep}^{\gamma}
: \R \to \R^n$ defined as
\begin{equation}\label{eq:Combination}
    G_{\ep}^{\gamma}(t) := \gamma \FF_{\ep}(t) + (1-\gamma)F_{\ep}(t) =
    F_{\ep}(t) + \gamma D_{\ep}(t).
\end{equation}
Clearly, if $\gamma \in [0,1]$, then $\{G_{\ep}^{\gamma} : \R \to \R^n \mid
\gamma \in [0,1]\}$ is the family of curves that arise as a convex combination
of $\FF_{\ep}$ with $F_{\ep}$ for a given $\ep > 0$. That is, by constructing
$G_{\ep}^{\gamma}$ as above, we obtain a path homotopy of the conventional and
directional mollifications. This allows, from a practical standpoint, to
generate an infinite family of curves for which the conventional and directional
mollifications are particular cases. If this family inherits the smoothness and
convergence properties of the previous methods, then it significantly raises its
potential applications. The following proposition ensures that we do not lose
any smoothness or convergence properties when combining the conventional and
directional mollifications.
\begin{proposition}
    Let $f : \R \to \R^n$ satisfy (for each component) the assumptions
    of Theorem \ref{thm:TheoremThatSolvesNewRegularization} and its statements \ref{thm:SolvReg1} and \ref{thm:SolvReg2}. Let $\varphi
    \in \DD_E(\R)$ and $\ep > 0$. Then for
    each $\gamma \in \R$ the function $G_{\ep}^{\gamma}$ defined
    as in \eqref{eq:Combination} satisfies that $G_{\ep}^{\gamma} \in
    C^{\infty}(\R,\R^n)$, and
    $G_{\ep}^{\gamma} \to f$ as $\ep \to 0^+$ pointwise and on compact
    subsets of $\R$.
\end{proposition}
\begin{proof}
    The result follows noting that $C^{\infty}(\R,\R^n)$ is a vector space,
    and by Theorem \ref{thm:TheoremThatSolvesNewRegularization} we have
    that $\FF_{\ep},F_{\ep} \in C^{\infty}(\R,\R^n)$ and $\FF_{\ep} \to f$,
    $F_{\ep} \to f$ pointwise, and uniformly on compact subsets of $\R$ as $\ep \to 0^+$.
    That is $G_{\ep}^{\gamma} \to \gamma f + (1-\gamma)f = f$ pointwise,
    and uniformly on compact subsets of $\R$ as $\ep \to 0^+$.
\end{proof}
It is clear that due to Theorem \ref{thm:TheoremThatSolvesNewRegularization} and
Lemma \ref{lem:DirectionalDerivativeAndNormalDerivative} the derivatives of
$G_{\ep}^{\gamma}$ are easy to compute. Let us now proceed with a concrete
example. Suppose that $f$ is defined as in \eqref{eq:polygonalchain}. First we
present an extremely versatile proposition, similar to Theorem
\ref{thm:ConvergenceToLinesToSwitchPolygonalChain}.
\begin{proposition}\label{prop:CombinationEqualAtHalf}
    Let $f$ be as in Definition \ref{def:polygonalchain}, $\varphi \in
    \DD_E(\R)$ and $\ep >0$. For $\gamma \in \R$ we
    consider $G_{\ep}^{\gamma} := \gamma \FF_{\ep} + (1-\gamma)F_{\ep} =
    F_{\ep} +
    \gamma D_{\ep}$,
    where $F_{\ep}$ and $D_{\ep}$ are the functions defined in
    Theorem \ref{thm:TheoremThatSolvesNewRegularization} for the extension
    $\bar{f}$ as in \eqref{eq:polygonalchainextended}.
    Suppose that  $\ep \in (0,\frac{1}{2})$, let $r \in \{1,\dots,p\}$ and define
    $V_r:= [r-1+\ep,r-\ep]$.
    Then it holds that $f|_{V_r} = F_{\ep}|_{V_r} = \FF_{\ep}|_{V_r} =
    G_{\ep}^{\gamma}|_{V_r}$ for all $ \gamma \in \R$.
\end{proposition}
\begin{proof}
    We can proceed in the same manner as in the proof of
    Theorem \ref{thm:ConvergenceToLinesToSwitchPolygonalChain}, noting from the proof of
    Theorem \ref{thm:TheoremThatSolvesNewRegularization} that if $f$ is an affine
    function then $D_{\ep}$ is identically zero, and so is $\gamma D_{\ep}$
    for any $\gamma \in \R$.
\end{proof}
\begin{remark}
    The same conclusions as in Remark \ref{rmk:VersatileApproach} hold.
    That is, at the sets where the original function $f$ and any element
    of the family of functions $\{G_{\ep}^{\gamma}\}_{\gamma\in\R}$ are
    equal, we can switch
    from one member of this family to another, obtaining a new function
    which is smooth. This includes $\gamma \in \{0,1\}$ which is the
    particular result of Theorem \ref{thm:ConvergenceToLinesToSwitchPolygonalChain} and
    Remark \ref{rmk:VersatileApproach}.
\end{remark}
Moreover, given $\gamma \in \R$ we are interested in a similar upper bound on the
curvature for $G_{\ep}^{\gamma}$ as the one presented in
Proposition \ref{prop:CurvatureBoundingOutside}.
\begin{proposition}\label{prop:CurvatureBoundingCombination}
    Let $f$ be as in Definition \ref{def:polygonalchain}, for $p=2$,
    and consider its extension $\bar{f}$. For any $\varphi \in
    \DD_E(\R)$, $\ep >0$ and $\gamma \in \R$ we consider $G_{\ep}^{\gamma} :=
    \gamma \FF_{\ep} + (1-\gamma)F_{\ep} = F_{\ep} +
    \gamma D_{\ep}$,
    where $F_{\ep}$ and $D_{\ep}$ are the functions defined in
    Theorem \ref{thm:TheoremThatSolvesNewRegularization} for $\bar{f}$.
    Then, if $0 \notin \{s\PP_1 + (1-s)\PP_2 \mid s \in \R\}$,
    the absolute value of the curvature of $G_{\ep}^{\gamma}$, defined as
    \begin{equation*}
        t \in \R \mapsto |\kappa(t)| = \frac{\left\|\left(G_{\ep}^{\gamma}\right)''(t)\wedge
    \left(G_{\ep}^{\gamma}\right)'(t)\right\|_2}{
    \left\|\left(G_{\ep}^{\gamma}\right)'(t)\right\|_2^3},
    \end{equation*}
    can be bounded by
    \begin{equation}\label{eq:CurvatureUpperBoundCombination}
        |\kappa(t)| \leq \left(|\gamma|\frac{\|\varphi'\|_{\infty}}{\ep^2}
        + |1+\gamma|\frac{\|\varphi\|_{\infty}}{\ep}\right)
        \frac{\|\PP_2\wedge\PP_1\|_2}{\|\bar{s}\PP_1 + (1-\bar{s})\PP_2\|_2^3},
    \end{equation}
    for any $t \in \R$,
    where $\bar{s}$ and $\PP_i$ are defined as in
    Proposition \ref{prop:CurvatureBoundingOutside}.
\end{proposition}
\begin{proof}
    From Proposition \ref{prop:CurvatureBoundingOutside} and its proof, we know that (cf. Proposition \ref{prop:CurvatureBoundingOutside} for
    the notation), for any $t \in \R$,
    $F_{\ep}'(t) = \PP_1A_1(t) + \PP_2A_2(t)$, $F_{\ep}''(t) =
    \varphi_{\ep}(t-1)(\PP_2-\PP_1)$, $D_{\ep}'(t) =
    (t-1)\varphi_{\ep}(t-1)(\PP_2-\PP_1)$, $D_{\ep}''(t) =
    \left(\varphi_{\ep}(t-1) +
    (t-1)\varphi_{\ep}'(t-1)\right)(\PP_2-\PP_1)$, and $A_1(t)+A_2(t)=1$.
    Thus
    \begin{align*}
        \left({G_{\ep}^{\gamma}}\right)'(t) = \PP_1(A_1(t) -
        \gamma(t-1)\varphi_{\ep}(t-1))
        + \PP_2
        (A_2(t)+\gamma(t-1)\varphi_{\ep}(t-1)),
    \end{align*}
    and note that the coefficients of $\PP_1$ and $\PP_2$ sum to 1.
    Therefore $\left\|\left({G_{\ep}^{\gamma}}\right)'(t)\right\|_2 \geq \min_{s\in \R}\|s
    \PP_1 +
    (1-s)\PP_2\|_2 > 0$ for any $t \in \R$,
    which is the same optimization problem as in
    \eqref{eq:ConvexOptProblem}, and therefore $\bar{s}$ (cf.
    Proposition \ref{prop:CurvatureBoundingOutside}) is the minimum. Moreover, by some
    computations and noting that $A_1(t) +A_2(t) = 1$, we have that
    \begin{equation*}
        \left({G_{\ep}^{\gamma}}\right)''(t) \wedge
    \left({G_{\ep}^{\gamma}}\right)'(t) = (\gamma (t-1)\varphi_{\ep}'(t-1) +
    \varphi_{\ep}(t-1)(1+\gamma))(\PP_2\wedge \PP_1),
    \end{equation*}
    which implies that
    \begin{align*}
        |\kappa(t)| &= \frac{\left\|\left({G_{\ep}^{\gamma}}\right)''(t) \wedge
            \left({G_{\ep}^{\gamma}}\right)'(t)\right\|_2}{\left\|\left({G_{\ep}^{\gamma}}\right)'(t)\right\|_2^3}
        \leq
        \left(|\gamma|\frac{\|\varphi'\|_{\infty}}{\ep^2} +
        |1+\gamma|\frac{\|\varphi\|_{\infty}}{\ep}\right)
        \frac{\|\PP_2\wedge\PP_1\|_2}{\|\bar{s}\PP_1+(1-\bar{s})\PP_2\|_2^3}, \quad \forall t \in \R.
    \end{align*}
\end{proof}
\begin{remark}
    Observe that if $\gamma = 1$ in \eqref{eq:CurvatureUpperBoundCombination}, we recover the upper bound of \eqref{eq:CurvatureBoundedOutsideMollifier}, validating the results. Moreover, again, by choosing $\ep \in (0,\frac{1}{2})$ a similar approach as in Lemma \ref{lem:CurvatureApproachPractical} and Remark \ref{rmk:PracticalApproach} can be used to obtain an upper bound on the curvature of the complete curve.
\end{remark}

\begin{remark}[Practical guidelines for selecting $\gamma$]\label{rmk:ChoiceOfGamma}
    As established in Proposition \ref{prop:CurvatureBoundingCombination}, the
    analytical upper bound for the curvature $|\kappa|$ is explicitly
    controlled by the interaction between the smoothing parameter $\ep$
    and the parameter $\gamma$. For a fixed $\ep > 0$, the
    factor
    \begin{equation*}
    \left(|\gamma|\frac{\|\varphi'\|_{\infty}}{\ep^2} +
    |1+\gamma|\frac{\|\varphi\|_{\infty}}{\ep}\right)
    \end{equation*}
    is strictly increasing with respect to nonnegative values of $\gamma$,
    attaining its global minimum over the domain $\gamma \geq 0$ at $\gamma =
    0$. This minimum corresponds to the conventional mollification,
    while setting $\gamma = 1$ yields the directional mollification approach.
    Because the curvature upper bound increases monotonically as a function of
    $\gamma$, we propose the following practical rules-of-thumb:
    \begin{itemize}
        \item Strict vertex interpolation ($\gamma = 1$): If exact
        vertex intersection is a mandatory constraint of the application, the
        parameter must be set to $\gamma = 1$, which corresponds to the
        directional mollification. In this scenario, the smoothing
        parameter $\ep > 0$ should be tuned (using the strategies
        outlined in Remark \ref{rmk:PracticalApproach}) to keep the resulting
        curvature below the maximum allowable threshold.

        \item Not strict vertex interpolation but convex hull confinement
        ($\gamma = 0$): If exact vertex interpolation is not strictly required
        and the resulting curve must be confined into the convex hull of the
        vertices, $\gamma = 0$---that is, the conventional mollification---must be selected. Moreover, this value of $\gamma$ minimizes the
        analytical curvature bound that must be tuned to fit the requirements
        using $\ep>0$.

        \item Blending
        ($\gamma \in (0,1)$): If vertex interpolation and convex hull
        confinement are not required, $\gamma$ can be chosen freely
        between $0$ and $1$. In this case $G_{\ep}^{\gamma}$ acts as a
        convex combination of the conventional and directional mollification.
        While exact vertex intersection is lost for any $\gamma < 1$, values
        closer to $1$ pull the smoothed curve significantly toward the original
        vertices than conventional mollification ($\gamma=0$), thereby minimizing the
        tracking error. Moreover, the analytical upper bound on the curvature is
        highly sensitive to $\ep \to 0$, where it behaves as $O(\ep^{-2})$. In
        contrast, letting the parameter $\gamma \to 1$ yields---for fixed
        $\ep$---purely linear scaling growth in the curvature bound.
        Consequently, for a pre-specified curvature upper bound threshold and
        not a vertex intersection restriction, an effective strategy is to
        first construct the conventional mollification ($\gamma=0$) with an appropriate, non-overlapping $\ep > 0$, and subsequently increase $\gamma \to
        1$ up to the marginal limit of the curvature bound. This approach
        successfully decreases the spatial tracking error to the vertices
        without violating the curvature restrictions.
        \end{itemize}
    \end{remark}

For a comparison between spline methods and the representation of the set
\begin{equation*}
    \{G_{\ep}^{\gamma} \mid \gamma \in \{-1, -0.5, 0, 0.5, 1, 1.5,2\}\},
\end{equation*}
where the original function $f$ is defined in Definition
\ref{def:polygonalchain}, see Figure \ref{fig:ChangingGamma}. As can be seen, a
family of curves is generated for which the conventional and directional
mollifications are just special cases of the set ($\gamma = 0$ and $\gamma = 1$,
respectively). Moreover, note the results of Proposition
\ref{prop:CombinationEqualAtHalf} hold; all the curves are equal at the same
sets between knots, as their errors are equal to zero.
Additionally, Figure \ref{fig:ChangingGamma} illustrates how conventional and
directional mollification approaches behave in comparison to standard
spline-based methods within this representative example. While standard spline
schemes are highly optimized and widely effective for general-purpose curve
fitting, the directional mollification---and its extensions
via $\gamma$---offers a dedicated alternative for scenarios demanding formal and
analytical guarantees. Specifically, this approach naturally ensures $C^\infty$
smoothness, exact vertex interpolation if $\gamma = 1$, convex hull confinement
if $\gamma = 0$, analytical curvature guarantees for a polygonal chain, and low
computational complexity. Thus,  our method provides a flexible,
theoretically rigorous framework tailored for applications where at least one of
these specific geometric and analytical properties must be strictly guaranteed.

\begin{figure}[pos=t]
    \centering
    \includegraphics[width=\linewidth]{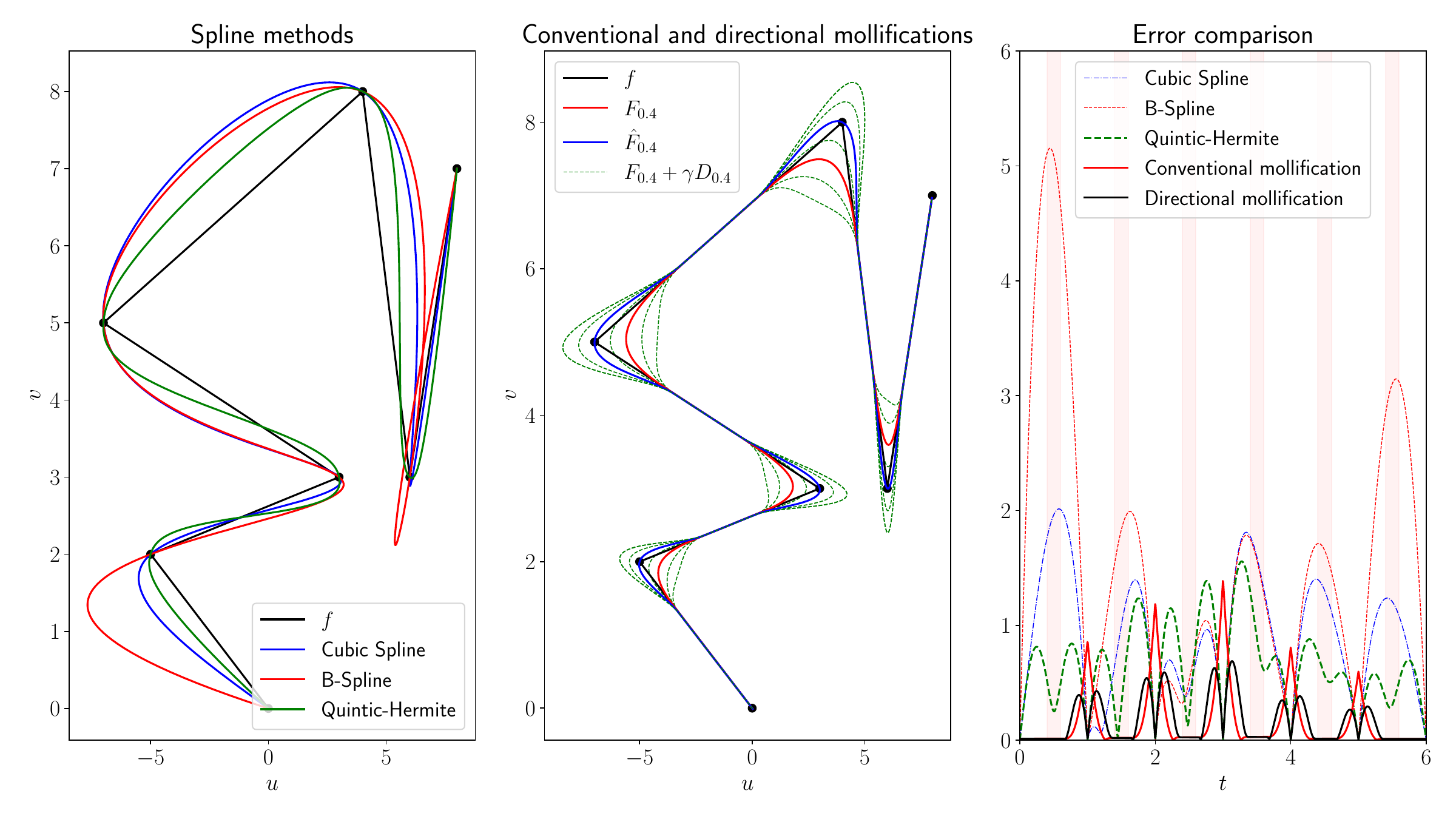}
    \caption{Illustrative comparison of a family of curves using
    \eqref{eq:Combination}, where the mollifier used is the one presented
    in Example \ref{example:OurMollifier} with $\ep = 0.4$ and different
    approaches of polynomial splines. The left picture represents
    in black a $p=6$ polygonal chain curve in $\R^2$ whose vertices are
    represented
    as black dots, and the components of $\R^2$ as $(u,v)$.
    It also shows in blue a cubic spline, in red a B-spline, and
    in green a Quintic Hermite polynomial spline. All of them are created such
    that they intersect the knots/vertices. The center picture also displays the original function as a black plot, while the red and blue curves show the conventional and directional mollifications of $f$, respectively.
    The dashed green lines represent $G_{\ep}^{\gamma}$ as in
    \eqref{eq:Combination}, for the values $\gamma \in
    \{-1,-0.5,0,0.5,1,1.5,2\}$. The right picture represents the Euclidean norms
    between the original function and the different approaches. That is, it
    represents the function $t \mapsto \|h(t)-f(t)\|_2$, where, $h$ is either, the
    cubic spline (blue dashed plot), B-spline (red dashed plot), quintic Hermite
    spline (green dashed plot), the conventional mollification (red solid plot)
    or the directional mollification (black plot) shown in the left and center
    pictures. The red solid regions represent the sets $V_r$ of Theorem
    \ref{thm:ConvergenceToLinesToSwitchPolygonalChain} for $r \in \{1,\dots,p\}$.}
    \label{fig:ChangingGamma}
\end{figure}
\section{Conclusions}\label{sec:Conclusions}

In this paper, we introduced a directional mollification operator that
transforms polygonal chains into $C^{\infty}$ curves while exactly interpolating
prescribed knots/vertices. By adding a convolution-weighted directional
derivative term to the classical smoothing, this method avoids the convex-hull
constraint of conventional mollifiers and preserves knots/vertices locations
without iterative solvers or knot insertion. The operator yields smooth
approximants that converge pointwise and uniformly on compact sets to the
original piecewise-linear curve as the smoothing parameter tends to zero, and
admits closed-form, easily computed curvature bounds. It features strong
locality: changing a single segment modifies the smoothed output only on that segment and within a controllably small neighborhood of its endpoints. Framing
the construction within a one-parameter family of operators unifies conventional
and directional mollification as special cases, providing a simple, efficient
method for curve design, trajectory planning, and CNC toolpath generation, where exact knot fidelity, smoothness, and explicit curvature control are required.

\section*{Use of AI tools declaration}
The authors declare they have not used Artificial Intelligence (AI) tools in the
creation of this article.

\section*{Acknowledgments}
The authors would like to thank the anonymous reviewers for their valuable
comments and constructive suggestions, which have significantly improved the
quality and presentation of this manuscript.

González-Calvin's work is especially supported by the FPU program of the Ministry of Science,
Innovation and Universities of Spain. This work is also supported by iRoboCity2030-CM,
Ref TEC-2024/TEC-62, financed by Comunidad Autónoma de Madrid (Spain) and by the
ERC Starting Grant iSwarm 101076091 and the RYC2020-030090-I grant from the
Spanish Ministry of Science.

\printcredits

\bibliographystyle{cas-model2-names}

\bibliography{cas-refs}

\end{document}